\documentclass{article}

\usepackage[preprint]{neurips_2026}

\usepackage[utf8]{inputenc} 
\usepackage[T1]{fontenc}    
\usepackage{hyperref}       
\usepackage{url}            
\usepackage{booktabs}       
\usepackage{amsfonts}       
\usepackage{nicefrac}       
\usepackage{microtype}      
\usepackage{xcolor}         

\usepackage{tabularx}
\usepackage{array}

\usepackage{amsmath,amssymb,amsthm,mathtools}
\usepackage{bm}
\usepackage{multirow}
\usepackage{algorithm}
\usepackage{algpseudocode}
\usepackage{caption}
\usepackage[nameinlink,noabbrev]{cleveref}

\newtheorem{theorem}{Theorem}
\newtheorem{lemma}{Lemma}
\newtheorem{proposition}{Proposition}
\newtheorem{corollary}{Corollary}
\newtheorem{definition}{Definition}
\newtheorem{assumption}{Assumption}
\theoremstyle{remark}\newtheorem{remark}{Remark}

\crefname{assumption}{Assumption}{Assumptions}
\Crefname{assumption}{Assumption}{Assumptions}
\crefname{definition}{Definition}{Definitions}
\Crefname{definition}{Definition}{Definitions}
\crefname{lemma}{Lemma}{Lemmas}
\Crefname{lemma}{Lemma}{Lemmas}
\crefname{proposition}{Proposition}{Propositions}
\Crefname{proposition}{Proposition}{Propositions}
\crefname{theorem}{Theorem}{Theorems}
\Crefname{theorem}{Theorem}{Theorems}
\crefname{corollary}{Corollary}{Corollaries}
\Crefname{corollary}{Corollary}{Corollaries}
\crefname{algorithm}{Algorithm}{Algorithms}
\Crefname{algorithm}{Algorithm}{Algorithms}

\newcommand{\E}{\mathbb{E}}
\newcommand{\Prob}{\mathbb{P}}
\newcommand{\PL}{\mathrm{PL}}
\newcommand{\TV}{\mathrm{TV}}
\newcommand{\rob}{\mathrm{rob}}
\newcommand{\nom}{\mathrm{nom}}
\DeclareMathOperator*{\argmin}{arg\,min}
\DeclareMathOperator*{\argmax}{arg\,max}
\DeclareMathOperator{\prox}{prox}
\DeclareMathOperator{\dist}{dist}
\DeclareMathOperator{\conv}{conv}
\newcommand{\Sk}{S_K}

\title{Distributionally Robust Listwise Preference Optimization}

\author{%
  Xudong Wu \\
  The University of Hong Kong \\
  Hong Kong SAR 
  \And
  Jian Qian \\
  The University of Hong Kong \\
  Hong Kong SAR
  \And
  Pangpang Liu \\
  Yale University \\
  New Haven, CT, USA
  \And
  Vaneet Aggarwal \\
  Purdue University \\
  West Lafayette, IN, USA
  \And
  Jiayu Chen \\
  The University of Hong Kong \\
  Hong Kong SAR
}

\begin{document}

\maketitle

\begin{abstract}
Existing robust preference optimization for language-model alignment mainly studies pairwise
supervision and places robustness at the dataset, prompt, or preference-pair level. We instead study
listwise preference optimization under ranking-label uncertainty: given a prompt and a candidate list, the observed ranking over that list may be ambiguous due to annotator inconsistency,
near-ties, lossy rankwise feedback, or reward-model noise. We propose a pointwise
total-variation robust Plackett--Luce objective that directly robustifies the ranking label conditional
on the candidate list. The robust loss admits an exact decomposition into the nominal PL loss plus a
worst-case PL correction, and the worst-case ranking is obtained by sorting current implicit scores
in ascending order, reducing the inner maximization from $K!$ enumeration to $O(K\log K)$. This tractable structure yields strong offline and online optimization guarantees. In the offline
fixed-list setting, the robust objective is convex and projected stochastic subgradient reaches global $\epsilon$-suboptimality with $O(\epsilon^{-2})$
sample complexity. In the online policy-induced setting, where candidate
lists are generated by the current policy, we establish weak convexity and
$\widetilde O(\epsilon^{-2})$ Moreau-envelope stationarity. Experiments in offline LLM alignment show that the proposed robust correction largely preserves
performance under clean labels and improves robustness under noise. In online alignment, it makes reward-model-ranked candidate
expansion more reliable and improves both reward-model and external GPT-4 judge metrics.
\end{abstract}

\section{Introduction}
\label{sec:intro}
Learning from human preferences has become a central mechanism for aligning large language models,
with reinforcement-learning-from-human-feedback (RLHF) and direct preference optimization (DPO)
serving as standard recipes for instruction-tuned models~\citep{christiano2017deep,ouyang2022training,rafailov2023direct}.
Most existing analyses and algorithms formulate preference feedback as \emph{pairwise} supervision
under a Bradley--Terry (BT) model~\citep{bradley1952rank,rafailov2023direct}, where each training
example identifies a chosen response and a rejected response.
However, many modern preference datasets are naturally \emph{listwise}: for a single prompt,
multiple candidate responses are available, and the supervision may contain a full or partial ranking.
This has motivated Plackett--Luce (PL) listwise preference objectives~\citep{plackett1975analysis,xia2008listmle,liu2025lipo,song2024preference}, which exploit the relative ordering of multiple
candidates rather than reducing the feedback to isolated pairwise comparisons.

While these works establish listwise preference optimization,
they do not address robustness to uncertainty in the observed ranking label itself. To the best of our
knowledge, this is the first work to study robust listwise preference optimization for LLM alignment
under conditional ranking-label ambiguity.

Robust preference optimization has developed largely along a different axis. 
Recent robust DPO-style methods introduce distributional uncertainty over the empirical preference dataset, the prompt distribution, or the preference-pair distribution~\citep{wu2025drdpo,mandal2025tvdpo,xu2026robustllmalignmentdistributionally}. 
These approaches address an important question: how should alignment behave when the distribution from which prompts or preference pairs are sampled is perturbed? 
They do not directly address a different and common source of uncertainty in listwise supervision: even after conditioning on the same prompt and the same candidate list, the observed ranking label itself may be unreliable.

This paper studies this \emph{conditional ranking-label uncertainty}. 
Given a prompt \(x\), a realized candidate list \(Y=(y_1,\ldots,y_K)\), and an observed ranking \(\sigma^\star\), we allow the ranking-label distribution to vary within a pointwise total-variation ambiguity set around the empirical ranking. 
This models local ambiguity caused by annotator inconsistency, near-ties between candidates, tied or lossy rankwise feedback, and reward-model-induced ranking noise in online alignment. 
The key distinction from prior robust preference-optimization work is that we do not robustify which prompts, pairs, or candidate lists are sampled; instead, conditional on a realized candidate list, we robustify the ranking label over that list.

This candidate-list-conditioned formulation has two advantages. 
First, it targets the supervision noise directly: for a fixed prompt--candidate-list instance, annotator disagreement, reward-model errors, near-ties, and lossy rank annotations all manifest as uncertainty in the observed ordering among the same alternatives.

Second, the formulation preserves the full listwise ranking signal. The PL model is a strict listwise generalization of the pairwise BT/DPO objective: when \(K=2\), it recovers the standard pairwise preference loss, while for \(K>2\), it retains the sequence-level ordering among multiple candidates. Prior listwise ranking and preference-optimization methods have shown that using the relative ordering among multiple candidates can exploit richer supervision than reducing feedback to isolated pairwise comparisons~\citep{xia2008listmle,song2024preference,liu2025lipo}. 
Our robustification keeps this listwise structure intact: it perturbs the ranking over the same candidate set rather than decomposing the supervision into independent pairwise label flips.

At first glance, this formulation appears computationally expensive because the adversary may choose among \(K!\) possible rankings. 
Our main structural observation is that the PL loss makes this inner maximization exactly tractable: the worst-case ranking is obtained by sorting the current implicit scores in ascending order. 
Consequently, the robust listwise loss reduces to a convex combination of the nominal PL loss and a single adversarial PL loss, computable in \(O(K\log K)\) time. 
For \(K\ge 3\), this correction is intrinsically listwise and does not reduce to a collection of independent pairwise BT corrections.

We analyze the resulting objective in both offline and online settings naturally induced by listwise alignment. 
In the offline fixed-list setting with log-linear scores, the robust objective is convex and reaches global $\epsilon$-suboptimality with $O(\epsilon^{-2})$
sample complexity. In the online policy-induced setting, where candidate lists are sampled from the current policy, the objective is no longer globally convex; we establish weak convexity and an \(\widetilde O(\epsilon^{-2})\) Moreau-envelope stationarity guarantee using an explicit ascending-sort Clarke-subgradient oracle. 
Empirically, the proposed correction behaves as a conservative ranking-label regularizer: it preserves clean-label performance, improves stability under structured ranking-label corruption, and helps larger listwise candidate sets become more reliable in online reward-model-driven alignment.

\subsection{Contributions}

Our contributions are as follows.

\begin{itemize}
    \item \textbf{A robust listwise preference-optimization formulation.}
    To the best of our knowledge, we are the first to study robust listwise preference optimization for
    LLM alignment under conditional ranking-label uncertainty. Unlike prior robust DPO methods that
    perturb the data, prompt, or pair distribution, our ambiguity set targets the ranking label conditional
    on a realized candidate list. The formulation recovers the standard pairwise BT/DPO setting when \(K=2\), while providing a genuinely listwise robustness model when \(K>2\).

    \item \textbf{Exact and tractable robust PL loss.}
    We show that the robust PL loss admits an exact decomposition into the nominal PL loss and a worst-case PL loss. 
    Although the inner maximization is over \(K!\) rankings, the PL structure implies that the worst-case ranking is simply the ascending-score order, giving an \(O(K\log K)\) evaluation algorithm.

    \item \textbf{Strong offline and online optimization theory.}
    In the offline fixed-list log-linear setting, the robust objective is convex and admits an \(O(\epsilon^{-2})\) stochastic subgradient guarantee. 
    In the online policy-induced setting, we prove weak convexity and an \(\widetilde O(\epsilon^{-2})\) Moreau-envelope stationarity bound using an explicit Clarke-subgradient oracle. These guarantees improve or match closely related robust
    preference-optimization rates under comparable log-linear oracle models, as summarized in
    Table~\ref{tab:theory_positioning}.

    \item \textbf{Empirical validation under ranking-label uncertainty.}
    Offline experiments show that robust PL is most useful when listwise labels are structurally corrupted, especially under severe top-rank noise. 
    Online LLM alignment experiments show that robustness helps larger candidate lists become more reliable when rankings are generated by a reward model.
\end{itemize}

\subsection{Relation to Prior Robust Preference Optimization}
\label{subsec:relation_prior}

Prior robust preference-optimization methods mainly study pairwise BT/DPO objectives and place robustness on the data, prompt, pair, or oracle distribution. 
In contrast, we study PL listwise preference optimization and place uncertainty on the ranking label conditional on a realized candidate list. 
This distinction is central: for \(K\ge 3\), our robust correction is a genuinely listwise PL max-gap over permutations, with an exact ascending-score solution rather than a reduction to independent pairwise BT corrections.

\Cref{tab:theory_positioning} summarizes how our setting differs from representative robust preference-optimization theories in terms of preference model, robustness source, robust loss structure, and sample/oracle complexity. 
Additional discussion is provided in Appendix \ref{app:related_work}.

\begin{table}[t]
\centering
\small
\caption{Comparison with representative robust preference-optimization theory.}
\label{tab:theory_positioning}
\setlength{\tabcolsep}{5pt}
\renewcommand{\arraystretch}{1.15}
\begin{tabularx}{\textwidth}{
    >{\raggedright\arraybackslash}p{0.24\textwidth}
    >{\raggedright\arraybackslash}X
    >{\raggedright\arraybackslash}X
    >{\raggedright\arraybackslash}X}
\toprule
\textbf{Aspect}
& \textbf{Distributionally Robust DPO ~\citep{mandal2025tvdpo}}
& \textbf{Oracle-robust online alignment~\citep{li2026oraclerobustonlinealignmentlarge}}
& \textbf{Ours} \\
\midrule

Preference model
& Pairwise BT
& Pairwise BT
& PL rankings, including BT as $K=2$ \\

Robustness source
& Data / prompt distribution shift
& Pairwise oracle perturbation
& Ranking-label perturbation \\

Robust loss structure
& Reweighting-based robust objective
& Exact pairwise sensitivity penalty
& Exact listwise sensitivity penalty, including pairwise as $K=2$ \\

Log-linear policy class
& Yes
& Yes
& Yes \\

Main guarantee
&
Offline robust DPO:
\(nT=\tilde O(\varepsilon^{-4})\)
&
Online Moreau stationarity:
\(\tilde O(\varepsilon^{-2})\)
&
Offline:
\(B_sT=O(\varepsilon^{-2})\);
Online:
\(\tilde O(\varepsilon^{-2})\) Moreau stationarity
\\

Empirical setting
& Offline
& -
& Offline and online \\

\bottomrule
\end{tabularx}

\vspace{0.5em}
\begin{minipage}{0.97\textwidth}
\footnotesize
\emph{Note.}
The Main guarantee row reports the total sample/oracle complexity under the corresponding
oracle model. Detailed comparisons with~\citet{mandal2025tvdpo} and~\citet{li2026oraclerobustonlinealignmentlarge} are provided
in \Cref{rem:offline-comparison,rem:online-comparison}.
\end{minipage}
\end{table}

\section{Background: From Pairwise DPO to Listwise PL-DPO}
\label{sec:background}

Given a prompt $x$ and two responses $y^+,y^-$, the Bradley--Terry (BT) model assumes
$\Prob(y^+\succ y^-\mid x) = \sigma\!\big(r^\star(x,y^+)-r^\star(x,y^-)\big)$, where $\sigma$ is the logistic sigmoid and $r^\star$ is the latent ground-truth reward~\citep{bradley1952rank}. Under KL-regularized RLHF the optimal policy admits the Gibbs form $\pi^\star(y\mid x)\propto\pi_{\mathrm{ref}}(y\mid x)\exp(r^\star(x,y)/\beta)$, so $r^\star(x,y)=\beta\log[\pi^\star(y\mid x)/\pi_{\mathrm{ref}}(y\mid x)]+\beta\log Z(x)$. The partition $\beta\log Z(x)$ cancels in the BT difference, motivating the implicit DPO score $g_\theta(x,y):=\beta\log[\pi_\theta(y\mid x)/\pi_{\mathrm{ref}}(y\mid x)]$ and the pairwise DPO loss $\ell_{\mathrm{DPO}}(\theta;x,y^+,y^-) = -\log\sigma(g_\theta(x,y^+)-g_\theta(x,y^-))$~\citep{rafailov2023direct}.

\paragraph{Listwise generalization via Plackett--Luce.}
For each prompt $x$, suppose we observe a realized candidate list $\mathcal{Y}=\{y_1,\dots,y_K\}$ and a deterministic empirical ranking $\sigma^\star\in\Sk$, where $\sigma^\star_i$ is the index of the response placed at rank $i$. The listwise analogue of BT is the Plackett--Luce model~\citep{plackett1975analysis}:
\begin{equation}
\Prob(\sigma^\star\mid x,\mathcal{Y}) = \prod_{i=1}^K \frac{\exp(r^\star(x,y_{\sigma^\star_i}))}{\sum_{j=i}^K \exp(r^\star(x,y_{\sigma^\star_j}))}.
\label{eq:pl-prob}
\end{equation}
Substituting the same DPO reparameterization, the partition $\beta\log Z(x)$ cancels at every stage of the product, yielding the listwise PL-DPO loss
\begin{equation}
\ell_{\PL}(\theta;x,y_{1:K},\sigma) = -\sum_{i=1}^K g_\theta(x,y_{\sigma_i}) + \sum_{i=1}^K \log\!\Big(\sum_{j=i}^K \exp(g_\theta(x,y_{\sigma_j}))\Big).
\label{eq:pl-loss}
\end{equation}
At $K=2$ and $\sigma=(1,2)$, \eqref{eq:pl-loss} reduces to $-\log\sigma(g_\theta(x,y_1)-g_\theta(x,y_2))$, recovering pairwise DPO exactly. For $K\ge 3$, $\ell_{\PL}$ aggregates information from $K-1$ stagewise PL choices and is strictly more informative than any single pairwise comparison drawn from the same list. A formal curvature view of this stagewise information aggregation is provided in \Cref{prop:pl_curvature_decomposition}.

The \emph{nominal} listwise objective on a fixed offline dataset $\mathcal{D}=\{(x,y_{1:K},\sigma^\star)\}$ is $J_{\nom}(\theta) := \E_{(x,y_{1:K},\sigma^\star)\sim\mathcal{D}}[\ell_{\PL}(\theta;x,y_{1:K},\sigma^\star)]$. We add ranking-label robustness next.

\section{Pointwise TV-Robust Listwise Objective}
\label{sec:robust-objective}

We treat each empirical ranking label $\sigma^\star$ as a Dirac point mass on $\Sk$ and let an adversary perturb this label distribution within a TV ball of radius $\rho\in[0,1]$.

\begin{definition}[Pointwise TV ambiguity]
\label{def:tv-set}
For sample $(x,y_{1:K},\sigma^\star)$,
\begin{equation}
\mathcal{U}_{\TV}(\delta_{\sigma^\star},\rho) := \big\{P\in\Delta(\Sk):\TV(P,\delta_{\sigma^\star})\le\rho\big\},
\quad \TV(P,Q)=\tfrac12\|P-Q\|_1.
\label{eq:tv-set}
\end{equation}
\end{definition}

\begin{definition}[Robust listwise objective]
\label{def:robust-loss}
The pointwise robust loss and offline robust objective are
\begin{align}
\ell_{\rob}(\theta;x,y_{1:K},\sigma^\star) &:= \max_{P\in\mathcal{U}_{\TV}(\delta_{\sigma^\star},\rho)} \E_{\sigma\sim P}[\ell_{\PL}(\theta;x,y_{1:K},\sigma)], \notag \\
J_{\rob}(\theta) &:= \E_{\mathcal{D}}[\ell_{\rob}(\theta;x,y_{1:K},\sigma^\star)].
\label{eq:robust-loss}
\end{align}
\end{definition}

\begin{lemma}[Exact pointwise TV decomposition]
\label{lem:tv-decomp}
For every sample $(x,y_{1:K},\sigma^\star)$, parameter $\theta$, and $\rho\in[0,1]$,
\begin{equation}
\ell_{\rob}(\theta;x,y_{1:K},\sigma^\star) = (1-\rho)\,\ell_{\PL}(\theta;x,y_{1:K},\sigma^\star) + \rho\max_{\sigma\in\Sk}\ell_{\PL}(\theta;x,y_{1:K},\sigma).
\label{eq:tv-decomp}
\end{equation}
Equivalently, $\ell_{\rob} = \ell_{\PL}(\sigma^\star)+\rho\big(\max_\sigma\ell_{\PL}(\sigma)-\ell_{\PL}(\sigma^\star)\big)$.
\end{lemma}

The full proof is in Appendix \ref{app:tv-proof}.

\section{Tractable Worst-Case Ranking}
\label{sec:worst-case}

By \Cref{lem:tv-decomp}, evaluating the pointwise robust loss requires solving
\[
\max_{\sigma\in\Sk}\ell_{\PL}(\theta;x,y_{1:K},\sigma),
\]
which is naively a maximization over \(K!\) possible rankings. The key structural observation is that, under the PL loss, the adversarial ranking is not arbitrary: it is obtained by placing low-score candidates before high-score candidates. Thus the inner maximization reduces to a single sorting operation.

\begin{theorem}[Worst-case ranking by ascending scores]
\label{thm:worst-perm}
Fix \((\theta,x,y_{1:K})\) and let \(g_i:=g_\theta(x,y_i)\). 
If \(\sigma_{\rm wc}\in\Sk\) sorts the scores in nondecreasing order,
\(g_{\sigma_{\rm wc}(1)}\le \cdots \le g_{\sigma_{\rm wc}(K)}\), then
\[
\sigma_{\rm wc}\in
\arg\max_{\sigma\in\Sk}
\ell_{\PL}(\theta;x,y_{1:K},\sigma).
\]
Hence the inner maximization in the robust PL loss is solved by sorting and costs \(O(K\log K)\).
With ties, any deterministic tie-breaking rule within tied groups is valid.
\end{theorem}

The full proof is in Appendix \ref{app:sort-proof}.

\begin{algorithm}[t]
\caption{Exact robust PL loss}
\label{alg:robust-pl-loss}
\begin{algorithmic}[1]
\Require \(x,Y=(y_1,\ldots,y_K),\sigma^\star,\theta,\rho\)
\State \(g_i\gets g_\theta(x,y_i)\) for \(i\in[K]\)
\State \(\sigma_{\rm wc}\gets \operatorname{argsort}_{i\in[K]}(g_i)\) in nondecreasing order
\State \(\ell_{\rm nom}\gets \ell_{\PL}(\theta;x,Y,\sigma^\star)\), \(\ell_{\rm wc}\gets \ell_{\PL}(\theta;x,Y,\sigma_{\rm wc})\)
\State \Return \((1-\rho)\ell_{\rm nom}+\rho\ell_{\rm wc}\)
\end{algorithmic}
\end{algorithm}

\begin{figure}[!t]
\centering
\scriptsize
\setlength{\abovecaptionskip}{2pt}
\setlength{\belowcaptionskip}{2pt}

\begin{minipage}[t]{0.49\textwidth}
\hrule
\vspace{2pt}
\captionof{algorithm}{Offline Robust PL-DPO}
\label{alg:offline-robust-pl-dpo}
\vspace{2pt}
\hrule
\vspace{2pt}
\begin{algorithmic}[1]
\Require Dataset \(\mathcal D=\{(x_i,Y_i,\sigma_i^\star)\}_{i=1}^n\), \(\theta_0\in\Theta\), stepsize \(\eta\), batch size \(B_s\), radius \(\rho\).
\For{\(t=0,\ldots,T-1\)}
    \State Sample mini-batch \(\mathcal B_t\subset\mathcal D\).
    \State \(\widehat{\mathcal L}_{\rm RPL}(\theta_t)
    \gets B_s^{-1}\!\!\sum_{(x_i,Y_i,\sigma_i^\star)\in\mathcal B_t}
    \ell_{\rm rob}(\theta_t;x_i,Y_i,\sigma_i^\star)\), with \(\ell_{\rm rob}\) from \Cref{alg:robust-pl-loss}.
    \State Choose \(\widehat g_t\in\partial_\theta
    \widehat{\mathcal L}_{\rm RPL}(\theta_t)\).
    \State \(\theta_{t+1}\gets \Pi_\Theta(\theta_t-\eta\widehat g_t)\).
\EndFor
\State \Return \(\bar\theta_T=T^{-1}\sum_{t=0}^{T-1}\theta_t\).
\end{algorithmic}
\vspace{2pt}
\hrule
\end{minipage}
\hfill
\begin{minipage}[t]{0.49\textwidth}
\hrule
\vspace{2pt}
\captionof{algorithm}{Online Robust PL-SAIL}
\label{alg:online-robust-pl-sail}
\vspace{2pt}
\hrule
\vspace{2pt}
\begin{algorithmic}[1]
\Require Initial parameter \(\theta_0\in\Theta\), stepsize \(\eta\), batch size \(B_s\), list size \(K\), radius \(\rho\).
\For{\(t=0,\ldots,T-1\)}
    \State Sample \(x_i\sim\mathcal D_x\), generate
    \(Y_i\sim\pi_{\theta_t}^{\otimes K}(\cdot|x_i)\), and obtain
    \(\sigma_i^\star\sim p^\star(\cdot|x_i,Y_i)\), for \(i\in[B_s]\).
    \State Evaluate \(\ell_{\rm rob}(\theta_t;x_i,Y_i,\sigma_i^\star)\)
    by \Cref{alg:robust-pl-loss}.
    \State Form \(\widehat G_t\gets B_s^{-1}\sum_{i=1}^{B_s}G(\theta_t;Z_i)\), where \(Z_i=(x_i,Y_i,\sigma_i^\star)\) and \(G\) is defined in \eqref{eq:oracle}.
    \State \(\theta_{t+1}\gets \Pi_\Theta(\theta_t-\eta\widehat G_t)\).
\EndFor
\State \Return \(\theta_R\), \(R\sim{\rm Uniform}\{0,\ldots,T-1\}\).
\end{algorithmic}
\vspace{2pt}
\hrule
\end{minipage}
\end{figure}

\section{Optimization Theory}
\label{sec:theory}

We analyze two settings: an \emph{offline} fixed-list setting in which $(y_{1:K},\sigma^\star)$ is independent of $\theta$, and an \emph{online} policy-induced setting in which $Y$ is sampled from the current policy $\pi_\theta^{\otimes K}$.

The offline and online optimization procedures are summarized in \Cref{alg:offline-robust-pl-dpo,alg:online-robust-pl-sail}.

\subsection{Offline Fixed-List Setting}
\label{sec:offline-theory}

\begin{assumption}[Log-linear policy class]
\label{ass:offline}
The response space \(\mathcal Y\) is finite. Let
\(\psi:\mathcal X\times\mathcal Y\to\mathbb R^{d_p}\) satisfy
\(\sup_{x,y}\|\psi(x,y)\|_2\le B_\psi\), where \(B_\psi=1\) can be obtained by rescaling. For
\(B>0\), let \(\Theta:=\{\theta\in\mathbb R^{d_p}:\|\theta\|_2\le B\}\), and consider
\[
\Pi
=
\left\{
\pi_\theta:
\pi_\theta(y|x)
=
\frac{\exp(\theta^\top\psi(x,y))}
{\sum_{y'\in\mathcal Y}\exp(\theta^\top\psi(x,y'))},
\ \theta\in\Theta
\right\}.
\]
Assume \(\pi_{\rm ref}=\pi_{\theta_{\rm ref}}\) for some fixed
\(\theta_{\rm ref}\in\Theta\), and set
\(D:=\sup_{\theta\in\Theta}\|\theta-\theta_{\rm ref}\|_2<\infty\).
\end{assumption}
\begin{remark}[On the log-linear policy class assumption]
The log-linear policy assumption is a standard simplification in theoretical analyses of reinforcement learning~\citep{agarwal2021theoryPG,modi2020sample}, RLHF~\citep{zhu2023principled}, and DPO~\citep{nika2024reward}. Closely related robust-alignment works also adopt log-linear policy assumption, including the online oracle-robust alignment setting of \citet{li2026oraclerobustonlinealignmentlarge}, the offline Distributionally Robust DPO setting of \citet{mandal2025tvdpo}, and distributionally robust DPO variants such as WDPO/KLDPO~\citep{xu2026robustllmalignmentdistributionally}. We use this assumption to isolate the optimization structure of the proposed robust listwise objective.
\end{remark}

\begin{proposition}[Convexity of the offline robust objective]
\label{prop:offline-convex}
Suppose \cref{ass:offline} holds, $\theta\mapsto\ell_{\PL}(\theta;x,y_{1:K},\sigma)$ is convex for every $(x,y_{1:K},\sigma)$. Hence by \Cref{lem:tv-decomp}, $\theta\mapsto\ell_{\rob}(\theta;x,y_{1:K},\sigma^\star)$ is convex on $\Theta$. The offline robust objective $J_{\rob}$ is convex.
\end{proposition}

The proof is in Appendix \ref{app:offline-convex}. We use the projected stochastic subgradient method
\begin{equation}
\theta_{t+1} = \Pi_\Theta\big(\theta_t - \eta\,\widehat g_t\big),\qquad t=0,\dots,T-1,
\label{eq:proj-subgrad}
\end{equation}
where $\widehat g_t$ is a mini-batch unbiased subgradient estimator with conditional variance bounded by $\sigma_g^2/B_s$ (mini-batch size $B_s$).

\begin{theorem}[Offline suboptimality of projected stochastic subgradient]
\label{thm:offline-rate}
Suppose \cref{ass:offline} holds, a mini-batch oracle with second moment bounded as $\E[\|\widehat g_t\|_2^2\mid\theta_t]\le 4K^2B_\psi^2+\sigma_g^2/B_s$, with $\eta = 2B/\sqrt{T(4K^2B_\psi^2+\sigma_g^2/B_s)}$ the averaged iterate $\bar\theta_T:=\tfrac1T\sum_{t=0}^{T-1}\theta_t$ satisfies
\begin{equation}
\E[J_{\rob}(\bar\theta_T)] - \min_{\theta\in\Theta} J_{\rob}(\theta) \le \frac{2B\sqrt{4K^2B_\psi^2+\sigma_g^2/B_s}}{\sqrt{T}}.
\label{eq:offline-rate}
\end{equation}
Consequently $T = O((4K^2B_\psi^2+\sigma_g^2/B_s)/\varepsilon^2)$ iterations suffice for $\varepsilon$-suboptimality. With fixed $B_s=\Theta(1)$, the total sample complexity is $B_s T=O(\varepsilon^{-2})$.
\end{theorem}

The full proof is in Appendix \ref{app:offline-rate}.

\begin{remark}[Comparison with prior robust DPO theory]
\label{rem:offline-comparison}
\Cref{thm:offline-rate} is most directly comparable to the DPO-side rate in offline
Distributionally Robust DPO~\citep{mandal2025tvdpo}. The
robustness source is different: Distributionally Robust DPO considers distributional
robustness over the data or prompt distribution, whereas our method considers pointwise
ranking-label robustness on a realized candidate list. This difference leads to different total sample/oracle complexities. In
\citet{mandal2025tvdpo}, the DPO-side guarantee uses
\(T=O(\varepsilon^{-2})\) iterations and a minibatch size satisfying
\(n/\log n=\Omega(\varepsilon^{-2})\), leading to total sample complexity
\(nT=\tilde O(\varepsilon^{-4})\). In contrast, our explicit listwise robust oracle yields \(T=O(\varepsilon^{-2})\), \(B_sT=O(\varepsilon^{-2})
\) when the minibatch size is fixed as \(B_s=\Theta(1)\). Thus, under our bounded
second-moment oracle model, the offline listwise robust objective admits a cleaner
\(O(\varepsilon^{-2})\) sample/oracle complexity.
\end{remark}

\begin{remark}[Optimality under the current model]
\label{rem:offline-rate-optimality}
This \(O(T^{-1/2})\) rate is the standard stochastic subgradient rate for nonsmooth convex objectives; faster rates would require additional smoothness, curvature, or non-degeneracy assumptions that stabilize the active worst-case ranking.
\end{remark}

\subsection{Online Policy-Induced Setting}
\label{sec:online-theory}

We now study the on-policy setting where, for each prompt $x\sim\mathcal{D}_x$, the candidate list $Y=(y_1,\dots,y_K)\in\mathcal{Y}^K$ is sampled iid as $Y\sim\pi_\theta^{\otimes K}(\cdot\mid x)$. The ranking oracle returns $\sigma^\star\sim p^\star(\cdot\mid x,Y)$, conditionally independent of $\theta$ given $(x,Y)$. \emph{Notation:} in the offline log-linear setting (\Cref{sec:offline-theory}) we wrote scores as $g_\theta(x,y)=\theta^\top\phi(x,y)$. In the online setting we instead parameterize the policy as $\pi_\theta(y\mid x)\propto\exp(\theta^\top\psi(x,y))$ and let $s_\theta(x,y):=\log[\pi_\theta(y\mid x)/\pi_{\mathrm{ref}}(y\mid x)]$ be the induced log-ratio score. Both reduce to affine functions of $\theta$. The online objectives are
\begin{equation}
J_{\nom}^{\mathrm{on}}(\theta) := \E_{x,Y\sim\pi_\theta^{\otimes K},\sigma^\star\sim p^\star}[\ell_{\PL}(\theta;x,Y,\sigma^\star)],
\quad
J_{\rob}^{\mathrm{on}}(\theta) := \E_{x,Y\sim\pi_\theta^{\otimes K},\sigma^\star\sim p^\star}[\ell_{\rob}(\theta;x,Y,\sigma^\star)].
\label{eq:online-objectives}
\end{equation}
By \Cref{lem:tv-decomp}, $J_{\rob}^{\mathrm{on}}=J_{\nom}^{\mathrm{on}}+\rho A(\theta)$ with $A(\theta)=\E\big[\max_\sigma\ell_{\PL}(\theta;x,Y,\sigma)-\ell_{\PL}(\theta;x,Y,\sigma^\star)\big]$. This connects the SAIL-style bilevel formulation of online alignment~\citep{ding2024sail} with our pointwise-TV listwise robustness; see \Cref{app:sail} for the bilevel-to-single-level reduction we use.

Because $\pi_\theta^{\otimes K}$ depends on $\theta$, $J_{\rob}^{\mathrm{on}}$ is \emph{not} globally convex in general. Our analysis therefore departs from the offline case: we work in the weakly convex framework~\citep{davis2019stochastic,drusvyatskiy2018error}, with the Clarke subdifferential $\partial_C$~\citep{clarke1990optimization}.

\begin{assumption}[Online policy-induced sampling and well-posedness]
\label{ass:online}
For each \(\theta\in\Theta\), draw \(x\sim\mathcal D_x\), sample
\(Y=(y_1,\ldots,y_K)\sim\pi_\theta^{\otimes K}(\cdot|x)\), and draw
\(\sigma^\star\sim p^\star(\cdot|x,Y)\), where \(\mathcal D_x\) is independent of \(\theta\) and
\(\sigma^\star\) is conditionally independent of \(\theta\) given \((x,Y)\). Define
\[
F(\theta):=J_{\rm rob}^{\rm on}(\theta)+I_\Theta(\theta),\qquad
I_\Theta(\theta)=
\begin{cases}
0,&\theta\in\Theta,\\
+\infty,&\theta\notin\Theta,
\end{cases}
\qquad
F_{\inf}:=\inf_{\theta\in\mathbb R^{d_p}}F(\theta)>-\infty .
\]
Assume \(F\) is proper and lower semicontinuous and bounded below.
\end{assumption}
We do not place the existence of a Clarke-subdifferential stochastic oracle as an assumption; instead we \emph{construct} one from ascending-sort below.


\begin{proposition}[Online weak convexity]
\label{prop:online-wc}
Suppose \cref{ass:offline,ass:online} hold. $J_{\rob}^{\mathrm{on}}$ is $\kappa$-weakly convex on $\Theta$ with
\begin{equation}
\kappa = K^2 B_\psi^2\big(8 + \log K + 2DB_\psi\big).
\label{eq:kappa-listwise}
\end{equation}
\end{proposition}

The bound \eqref{eq:kappa-listwise} is polynomial in $K$ and the model-boundedness constants; the full argument is deferred to Appendix \ref{app:online-proofs}.

\paragraph{Stochastic oracle from ascending-sort.}
Let \(\sigma^{\mathrm{sel}}(\theta;x,Y)\) be the deterministic ascending-score maximizer of
\(\ell_{\PL}(\theta;x,Y,\sigma)\), with fixed tie-breaking, as given by \Cref{thm:worst-perm}. For
\(Z=(x,Y,\sigma^\star)\), define
\begin{equation}
G(\theta;Z)
:=
(1-\rho)\nabla\ell_{\PL}(\theta;x,Y,\sigma^\star)
+\rho\nabla\ell_{\PL}(\theta;x,Y,\sigma^{\mathrm{sel}})
+\ell_{\rob}(\theta;x,Y,\sigma^\star)S_\theta(x,Y).
\label{eq:oracle}
\end{equation}
The last term is the score-function correction for policy-induced sampling. \Cref{lem:oracle}
verifies that \(G\) is a valid stochastic Clarke-subgradient oracle with bounded second moment.

\begin{theorem}[Online robust convergence]
\label{thm:online-rate}
Suppose \cref{ass:offline,ass:online} hold. Let $\kappa$ be as in \eqref{eq:kappa-listwise} and $G_{\mathrm{tot}}^2$ as in \eqref{eq:Gtot}.  For $F=J_{\rob}^{\mathrm{on}}+I_\Theta$ and $\hat\lambda\in(0,1/\kappa)$, define
$F_{\hat\lambda}(\theta)=\min_{u\in\Theta}\{J_{\rob}^{\mathrm{on}}(u)+\tfrac1{2\hat\lambda}\|u-\theta\|^2\}$. Fix $\hat\lambda\in(0,1/\kappa)$ and set $\Delta_0:=F_{\hat\lambda}(\theta_0)-F_{\inf}$. Run \Cref{alg:online-robust-pl-sail} with the explicit ascending-sort oracle \eqref{eq:oracle} and constant stepsize $\eta=\sqrt{2\hat\lambda\Delta_0/(G_{\mathrm{tot}}^2 T)}$. Then
\begin{equation}
\E\!\big[\|\nabla F_{\hat\lambda}(\theta_R)\|^2\big] \le \frac{2}{1-\kappa\hat\lambda}\sqrt{\frac{2\,\Delta_0\,G_{\mathrm{tot}}^2}{\hat\lambda\,T}}.
\label{eq:online-rate}
\end{equation}
\end{theorem}

\begin{corollary}[Sample / oracle complexity]
\label{cor:complexity}
$\E\|\nabla F_{\hat\lambda}(\theta_R)\|^2\le\varepsilon$ holds for $T\ge 8\Delta_0 G_{\mathrm{tot}}^2/[\hat\lambda(1-\kappa\hat\lambda)^2\varepsilon^2]$. At $\hat\lambda=1/(2\kappa)$,
\begin{equation}
T = O\!\left(\frac{\kappa\,\Delta_0\,G_{\mathrm{tot}}^2}{\varepsilon^2}\right) \;=\; \tilde O\!\left(\frac{K^6 B_\psi^4(\log K+DB_\psi)^3}{\varepsilon^2}\right).
\label{eq:online-complexity}
\end{equation}
\end{corollary}

The proofs of \Cref{thm:online-rate,cor:complexity} are in Appendix \ref{app:online-proofs}.

\begin{remark}[Comparison with oracle-robust online alignment~\citep{li2026oraclerobustonlinealignmentlarge}]
\label{rem:online-comparison}
\Cref{thm:online-rate} matches the \(\tilde O(\varepsilon^{-2})\) Moreau-stationarity order of
\citet{li2026oraclerobustonlinealignmentlarge}, but for a different robust object: they study
pairwise BT oracle perturbations, while we study listwise PL ranking-label perturbations, with BT
recovered at \(K=2\). Moreover, in our finite-response log-linear setting, the required PL bounds,
score-function bounds, weak-convexity constant, and stochastic Clarke oracle are derived explicitly
from the listwise structure rather than postulated as abstract regularity/oracle assumptions in~\citet{li2026oraclerobustonlinealignmentlarge}.
\end{remark}

\section{Experiments}
\label{sec:experiments}

\paragraph{Evaluation questions.}
Our experiments test the behavior predicted by the proposed ambiguity model rather than claiming that robustness monotonically improves all metrics. 
We ask whether the robust correction: 
(i) preserves performance when rankings are reliable; 
(ii) reduces the failure modes of PL learning when listwise labels are corrupted; and 
(iii) makes larger candidate lists more reliable in online alignment, where rankings are generated by a reward model.

We evaluate the proposed pointwise-TV robust PL objective in two settings matching our theory:
offline fixed-list ranking with corrupted labels, and online policy-induced alignment with
reward-model-generated rankings. Full experimental details are deferred to
Appendix \ref{app:exp-details}.

We provide anonymized code repositories for the online and offline experiments at
\url{https://anonymous.4open.science/r/robust-listwise-online-09BB}
and
\url{https://anonymous.4open.science/r/robust-listwise-offline-7FF6},
respectively.

\paragraph{Setup.}
We use UltraFeedback~\citep{cui2023ultrafeedback}, where each prompt has four candidate responses,
yielding a natural listwise preference problem with $K=4$. We compare three objectives: 
\emph{Nominal BT}, the standard pairwise DPO/BT baseline using chosen--rejected comparisons;
\emph{Nominal PL}, the non-robust listwise Plackett--Luce objective; and \emph{Robust PL}, our
pointwise-TV robust version of the PL objective. In the online setting, candidate responses are
generated by the current policy and ranked by a frozen reward model; the corresponding evaluation
metrics are defined in the online result table.

\subsection{Offline fixed-list evaluation}
\label{subsec:offline-exp}

The offline setting directly matches our ambiguity model: the prompt and candidate list are fixed,
while the ranking label may be corrupted. Since binary chosen--rejected label flipping has no unique
canonical analogue for a full ranking, we introduce two listwise corruptions. \emph{Near-tie}
corruption swaps the adjacent pair with the closest annotation scores, modeling local ambiguity
between nearly indistinguishable responses. \emph{Top-rank} corruption promotes a lower-ranked
response to the first position, modeling a more severe error because the first PL stage selects from
the full candidate list. The noise level $\epsilon$ is the fraction of corrupted training rankings;
evaluation labels are always clean.

We also include two pairwise robust-DPO baselines, TV-DR-DPO~\citep{mandal2025tvdpo} and
KLDPO~\citep{xu2026robustllmalignmentdistributionally}, reimplemented in the same pipeline using
their loss-level robust DPO objectives. Their hyperparameters are selected from held-out sweeps, with
the sweep results reported in Appendix~\ref{app:pairwise_dro_sweeps}. We evaluate all methods by Kendall's $\tau$ on clean held-out rankings. Each method assigns scalar
scores to the four candidate responses, which induce a predicted ranking; Kendall's $\tau$ measures
the rank correlation with the clean UltraFeedback reference ranking.

\paragraph{Offline observations.}
\Cref{tab:main_kendall_results} shows that Robust PL incurs only a small degradation under clean
labels while providing clear gains when the ranking labels are corrupted. The improvement is most
pronounced under severe top-rank corruption: when $\epsilon=1.0$, Robust PL substantially improves
over Nominal PL for both Qwen3-0.6B and Qwen3-8B. This matches the PL structure: top-rank errors
corrupt the early stagewise choices that dominate the likelihood, whereas near-tie corruption often
preserves much of the global ordering. Thus, in the offline fixed-list setting, Robust PL behaves as a
conservative ranking-label regularizer: it largely preserves clean-label ranking performance while
improving robustness to structured listwise label noise.

Appendix~\ref{app:additional-exp} provides further support: additional clean-label metrics show limited performance degradation, $\rho$-sweeps validate the robustness--over-regularization tradeoff, and Qwen2.5-0.5B/7B results show consistent trends across model families.

\begin{table}[t]
\centering
\caption{
Main offline ranking results under synthetic ranking-label corruption. 
We report Kendall's $\tau$ on clean held-out UltraFeedback rankings; higher is better. 
For each prompt, model scores induce a ranking over the four candidate responses, which is compared
with the clean reference ranking. The corruption level $\epsilon$ applies only to training labels.
}
\label{tab:main_kendall_results}
\resizebox{\textwidth}{!}{
\begin{tabular}{llcccccc}
\toprule
Model & Noise condition 
& Nominal BT (DPO) 
& TV-DR-DPO $(\rho=0.10)$ 
& KLDPO $(\tau=1.00)$ 
& Nominal PL 
& Robust PL $(\rho=0.05)$ 
& Robust PL $(\rho=0.10)$ \\
\midrule
\multirow{5}{*}{Qwen3-0.6B}
& Clean & \textbf{0.298} & 0.282 & 0.277 & 0.288 & 0.276 & 0.284 \\
& near\_tie, $\epsilon=0.4$ & 0.266 & 0.270 & 0.267 & 0.268 & \textbf{0.274} & 0.264 \\
& near\_tie, $\epsilon=1.0$ & 0.261 & 0.244 & 0.246 & 0.244 & \textbf{0.262} & 0.246 \\
& top\_rank, $\epsilon=0.4$ & 0.244 & 0.229 & \textbf{0.263} & 0.243 & 0.251 & 0.236 \\
& top\_rank, $\epsilon=1.0$ & 0.036 & 0.024 & 0.031 & 0.119 & \textbf{0.154} & 0.116 \\
\midrule
\multirow{5}{*}{Qwen3-8B}
& Clean & 0.316 & 0.338 & 0.331 & \textbf{0.362} & 0.347 & 0.340 \\
& near\_tie, $\epsilon=0.4$ & 0.291 & 0.319 & 0.318 & 0.348 & \textbf{0.356} & 0.333 \\
& near\_tie, $\epsilon=1.0$ & 0.281 & 0.310 & 0.310 & 0.322 & \textbf{0.327} & 0.319 \\
& top\_rank, $\epsilon=0.4$ & 0.302 & 0.282 & 0.296 & 0.321 & \textbf{0.329} & 0.293 \\
& top\_rank, $\epsilon=1.0$ & 0.039 & -0.056 & -0.016 & 0.103 & \textbf{0.146} & 0.076 \\
\bottomrule
\end{tabular}
}
\end{table}

\subsection{Online policy-induced alignment}
\label{subsec:online-exp}

We next evaluate the online setting, where the current policy generates candidate responses and a
frozen reward model ranks them to provide the training signal. This setting naturally introduces
ranking-label uncertainty: increasing the list size from $K=2$ to $K=4$ provides richer preference
information, but also requires the reward model to make finer-grained comparisons over more
candidates. Thus, larger $K$ is not automatically beneficial.

Our method is a robust listwise extension of the SAIL-style online preference-optimization pipeline.
We therefore use the setting-matched binary SAIL baseline, recovered by $K=2,\rho=0$. The
non-robust listwise extension is $K=4,\rho=0$, while $\rho>0$ isolates the effect of the proposed
robust ranking-label correction. We do not include PPO-style online RLHF or offline robust-DPO
baselines as direct comparisons because they optimize different signals or robustify different objects. For external evaluation, we follow the LLM-as-a-judge protocol~\citep{zheng2023judging} and use
GPT-4~\citep{openai2023gpt4} as the judge to compare model outputs against the dataset chosen
responses.

\paragraph{Online observations.}
\Cref{tab:qwen3_online_results} shows that simply increasing the candidate-list size is not
sufficient. The non-robust listwise variant $K=4,\rho=0$ does not consistently improve over the
binary baseline $K=2,\rho=0$, suggesting that larger lists provide richer preference information
but also introduce finer-grained reward-model ranking noise. Robustness mitigates this issue. With
$\rho>0$, the $K=4$ variants become more reliable: for Qwen3-0.6B, $K=4,\rho=0.02$ gives the
best reward-model preference and GPT-4 judge scores; for Qwen3-8B, robust $K=4$ variants achieve
the strongest reward-model performance and ranking-agreement metrics. These results support our
main interpretation that Robust PL helps convert larger candidate lists from a noisier supervision
source into useful listwise preference signal. The GPT-4 judge gains further suggest that the
improvement transfers beyond the reward model used for training.

\begin{table}[t]
\centering
\caption{
Online Qwen3 results on the U10 held-out evaluation set. 
$K=2,\rho=0$ is the binary SAIL baseline; $K=4,\rho=0$ is the non-robust listwise extension; 
$\rho>0$ gives the robust variant. 
$\Delta$Reward is the average reward-model gain over the SFT reference, and Rwd\% vs SFT is the
fraction of prompts where the reward model prefers the trained response to the SFT response. 
GPT\% vs Chosen and GPT+Tie\% compare the trained response with the dataset chosen response using
GPT-4 as an external judge. Top-1, Pairwise, and Kendall's $\tau$ measure agreement between the
model-induced ranking and the reward-model ranking over candidate lists. Higher is better for all
metrics.
}
\label{tab:qwen3_online_results}
\resizebox{\textwidth}{!}{
\begin{tabular}{llccccccc}
\toprule
Model & Method & $\Delta$Reward & Rwd\% vs SFT & GPT\% vs Chosen & GPT+Tie\% & Top-1 & Pairwise & Kendall $\tau$ \\
\midrule
\multirow{7}{*}{Qwen3-0.6B}
& SFT reference & 0.0 & 50.0\% & 5.9\% & 14.5\% & 0.315 & 0.586 & 0.155 \\
& RPL, $K=2$, $\rho=0.00$ & +252.5 & 61.7\% & 7.8\% & 15.8\% & 0.350 & 0.568 & 0.121 \\
& RPL, $K=2$, $\rho=0.02$ & +333.0 & 62.9\% & 10.2\% & 17.8\% & 0.373 & 0.591 & \textbf{0.172} \\
& RPL, $K=2$, $\rho=0.05$ & +189.4 & 56.2\% & 7.1\% & 16.9\% & 0.332 & 0.548 & 0.085 \\
& RPL, $K=4$, $\rho=0.00$ & +261.5 & 61.7\% & 5.9\% & 14.8\% & 0.363 & 0.585 & 0.152 \\
& RPL, $K=4$, $\rho=0.02$ & \textbf{+350.1} & \textbf{64.1\%} & \textbf{10.5\%} & \textbf{18.8\%} & \textbf{0.388} & \textbf{0.595} & 0.163 \\
& RPL, $K=4$, $\rho=0.05$ & +341.3 & 61.7\% & 9.4\% & 18.6\% & 0.371 & 0.570 & 0.125 \\
\midrule
\multirow{8}{*}{Qwen3-8B}
& SFT reference & 0.0 & 50.0\% & 24.6\% & 42.6\% & 0.246 & 0.516 & 0.026 \\
& RPL, $K=2$, $\rho=0.00$ & +482.1 & 67.2\% & 25.0\% & 43.8\% & 0.389 & 0.621 & 0.223 \\
& RPL, $K=2$, $\rho=0.02$ & +407.4 & 66.0\% & 25.0\% & 47.5\% & 0.382 & 0.630 & 0.238 \\
& RPL, $K=2$, $\rho=0.05$ & +584.6 & 69.5\% & 27.7\% & \textbf{48.8\%} & 0.393 & 0.624 & 0.225 \\
& RPL, $K=4$, $\rho=0.00$ & +467.0 & 65.6\% & 23.8\% & 45.9\% & 0.391 & 0.621 & 0.218 \\
& RPL, $K=4$, $\rho=0.02$ & +426.7 & 65.2\% & 24.2\% & 46.1\% & 0.401 & 0.624 & 0.226 \\
& RPL, $K=4$, $\rho=0.05$ & \textbf{+610.7} & \textbf{71.9\%} & \textbf{28.9\%} & 48.0\% & 0.413 & 0.632 & 0.243 \\
& RPL, $K=4$, $\rho=0.10$ & +484.0 & 68.4\% & 25.8\% & 46.5\% & \textbf{0.436} & \textbf{0.646} & \textbf{0.265} \\
\bottomrule
\end{tabular}
}
\end{table}

\section{Conclusion}

We introduced a listwise-native notion of preference uncertainty: a pointwise total-variation
ambiguity set on the ranking label over a realized candidate list, combined with the Plackett--Luce
listwise loss. The resulting robust PL objective has three key guarantees. First, the inner worst-case
ranking problem over $K!$ permutations is exactly solved by ascending-score sorting, giving
$O(K\log K)$ evaluation. Second, in the offline fixed-list log-linear setting, the robust objective is
convex and projected stochastic subgradient descent reaches global $\epsilon$-suboptimality with
$O(\epsilon^{-2})$ sample complexity. Third, in the online policy-induced setting, the objective is
weakly convex and admits $\widetilde O(\epsilon^{-2})$ Moreau-envelope stationarity. Empirically, Robust PL largely preserves performance under clean labels, while improving robustness
when the training rankings are corrupted, especially under severe top-rank corruption. In online
alignment, it makes $K=4$ candidate expansion more reliable under reward-model-generated rankings
and improves both reward-model and external GPT-4 judge metrics.

\paragraph{Limitations.}
Empirically,
the robustness radius $\rho$ must be tuned. Future work should study adaptive choices of
$\rho$, richer ambiguity sets, and larger-scale online alignment experiments.

\paragraph{Broader impact.}
Robust listwise alignment can reduce the influence of noisy or inconsistent rankings on LLM
behavior. The techniques studied are primarily methodological and analytical in nature. We do not foresee any immediate negative societal impact arising from this work.

\bibliographystyle{plainnat}
\bibliography{references}

\appendix
\renewcommand{\thetable}{\Alph{section}.\arabic{table}}
\renewcommand{\thefigure}{\Alph{section}.\arabic{figure}}
\renewcommand{\theequation}{\Alph{section}.\arabic{equation}}
\renewcommand{\thetheorem}{\Alph{section}.\arabic{theorem}}
\renewcommand{\thelemma}{\Alph{section}.\arabic{lemma}}
\renewcommand{\theproposition}{\Alph{section}.\arabic{proposition}}
\renewcommand{\thecorollary}{\Alph{section}.\arabic{corollary}}
\renewcommand{\thedefinition}{\Alph{section}.\arabic{definition}}
\renewcommand{\theassumption}{\Alph{section}.\arabic{assumption}}
\renewcommand{\theHtable}{app.\Alph{section}.\arabic{table}}
\renewcommand{\theHfigure}{app.\Alph{section}.\arabic{figure}}
\renewcommand{\theHequation}{app.\Alph{section}.\arabic{equation}}
\renewcommand{\theHtheorem}{app.\Alph{section}.\arabic{theorem}}
\renewcommand{\theHlemma}{app.\Alph{section}.\arabic{lemma}}
\renewcommand{\theHproposition}{app.\Alph{section}.\arabic{proposition}}
\renewcommand{\theHcorollary}{app.\Alph{section}.\arabic{corollary}}
\renewcommand{\theHdefinition}{app.\Alph{section}.\arabic{definition}}
\renewcommand{\theHassumption}{app.\Alph{section}.\arabic{assumption}}
\setcounter{theorem}{0}
\setcounter{lemma}{0}
\setcounter{proposition}{0}
\setcounter{corollary}{0}
\setcounter{definition}{0}
\setcounter{assumption}{0}
\setcounter{equation}{0}
\setcounter{table}{0}
\setcounter{figure}{0}

\section{Additional Related Work}
\label{app:related_work}

A related line of work studies noisy or corrupted pairwise preference labels in offline DPO-style training. 
cDPO~\citep{mitchell2023noisy} and rDPO~\citep{chowdhury2024provably} introduce correction mechanisms for binary preference flipping under pairwise BT/DPO supervision. 
These methods are offline fixed-pair approaches: they correct noisy chosen--rejected labels in a static preference dataset, rather than addressing the online policy-induced setting where candidate lists are sampled from the current policy. 
Their loss forms are also different from ours: cDPO/rDPO correct binary BT/DPO losses, whereas our pointwise-TV formulation optimizes a worst-case PL loss over full ranking labels.

Another related line of work studies distributionally robust preference optimization under perturbations of the empirical data distribution, prompt distribution, or preference distribution. 
For example, \citet{wu2025drdpo} formulate distributionally robust pairwise DPO under perturbations of the dataset distribution; \citet{mandal2025tvdpo} study TV-based ambiguity over the joint training distribution, including the prompt distribution, for pairwise DPO and policy optimization; \citet{xu2026robustllmalignmentdistributionally} consider Wasserstein- and KL-based ambiguity around the empirical preference distribution; and \citet{li2026oraclerobustonlinealignmentlarge} analyze an online oracle-robust alignment setting.

In contrast, we study the listwise Plackett--Luce (PL) setting and robustify a different object: the conditional ranking-label distribution given a candidate list. 
This distinction is substantive. 
When \(K\ge 3\), the resulting robust correction is intrinsically listwise: it involves a PL max-gap over permutations and admits an efficient ascending-sort solution, rather than reducing to a collection of independent pairwise BT corrections. 
Our framework therefore complements prior robust pairwise DPO methods while extending robustness analysis from BT preferences to PL rankings, covering both the offline fixed-list setting and the online listwise alignment setting.

\section{Notation and conventions}
\label{app:notation}

Throughout the appendix, \(\|\cdot\|\) denotes the Euclidean norm on the parameter space and
\(\|\cdot\|_{\mathrm{op}}\) denotes the spectral norm. We write \(\Sk\) for the symmetric group on
\(K\) symbols and \(\partial_C f\) for the Clarke subdifferential of \(f\). For convex functions,
\(\partial_C f\) reduces to the usual convex subdifferential. For weakly convex nonsmooth functions,
we use the Clarke subdifferential together with the standard Moreau-envelope calculus for proper
lower semicontinuous weakly convex objectives. Throughout, \(\ell_{\PL}\) is defined in
\eqref{eq:pl-loss} and \(\ell_{\rob}\) in \eqref{eq:robust-loss}.

\paragraph{Convention check (\(\beta\)).}
In the online analysis we use the unscaled log-ratio score
\[
s_\theta(x,y)
:=
\log\frac{\pi_\theta(y\mid x)}{\pi_{\mathrm{ref}}(y\mid x)}.
\]
Some DPO conventions instead use the scaled score
\[
g_\theta(x,y):=\beta s_\theta(x,y).
\]
All results below are stated for the unscaled convention. Under the scaled convention, the
stagewise score gaps are multiplied by \(\beta\), gradients by \(\beta\), and Hessians by
\(\beta^2\). Accordingly, the constants become
\[
C_L=K(\log K+2\beta DB_\psi),\qquad
C_G=2\beta KB_\psi,\qquad
C_H=\beta^2KB_\psi^2,
\]
and
\[
\kappa
=
K^2\beta^2B_\psi^2
\bigl(8+\log K+2\beta DB_\psi\bigr).
\]
\section{\texorpdfstring{Full proof of the pointwise TV decomposition (\Cref{lem:tv-decomp})}{Full proof of the pointwise TV decomposition (Lemma 1)}}
\label{app:tv-proof}

\begin{proof}
Write $\ell(\sigma):=\ell_{\PL}(\theta;x,y_{1:K},\sigma)$. We solve
$\max_{P\in\Delta(\Sk)}\sum_\sigma P(\sigma)\ell(\sigma)$ subject to $\TV(P,\delta_{\sigma^\star})\le\rho$.

Expanding the TV distance and using $\delta_{\sigma^\star}(\sigma^\star)=1$ and $\delta_{\sigma^\star}(\sigma)=0$ for $\sigma\ne\sigma^\star$,
\begin{align*}
\TV(P,\delta_{\sigma^\star}) &= \tfrac12\big(|P(\sigma^\star)-1|+\sum_{\sigma\ne\sigma^\star}P(\sigma)\big) \\
&= \tfrac12\big((1-P(\sigma^\star))+(1-P(\sigma^\star))\big) = 1-P(\sigma^\star).
\end{align*}
Hence $\TV\le\rho \iff P(\sigma^\star)\ge 1-\rho$. Let $\epsilon:=1-P(\sigma^\star)\in[0,\rho]$. Then $\sum_{\sigma\ne\sigma^\star}P(\sigma)=\epsilon$ and
\[
\E_{\sigma\sim P}[\ell(\sigma)] = (1-\epsilon)\ell(\sigma^\star) + \sum_{\sigma\ne\sigma^\star}P(\sigma)\ell(\sigma).
\]
For fixed $\epsilon$, the adversary maximizes by allocating all mass $\epsilon$ to $\argmax_{\sigma\ne\sigma^\star}\ell(\sigma)$, giving
\[
\E_{\sigma\sim P}[\ell(\sigma)] \le (1-\epsilon)\ell(\sigma^\star)+\epsilon\max_{\sigma\in\Sk}\ell(\sigma) = \ell(\sigma^\star)+\epsilon\cdot\Big(\max_{\sigma}\ell(\sigma)-\ell(\sigma^\star)\Big).
\]
Since $\max_\sigma\ell(\sigma)\ge\ell(\sigma^\star)$ always (as $\sigma^\star\in\Sk$), the parenthetical is $\ge 0$, so the linear function of $\epsilon$ is monotone non-decreasing and is maximized at $\epsilon=\rho$. Substituting yields \eqref{eq:tv-decomp}.
\end{proof}

\section{\texorpdfstring{Full proof of the worst-case sorting theorem (\Cref{thm:worst-perm})}{Full proof of the worst-case sorting theorem (Theorem 1)}}
\label{app:sort-proof}

\begin{proof}
Write $g_i:=g_\theta(x,y_i)$. From \eqref{eq:pl-loss},
\[
\ell_{\PL}(\theta;x,y_{1:K},\sigma) = -\sum_{i=1}^K g_{\sigma_i} + \sum_{i=1}^K \log\!\Big(\sum_{j=i}^K e^{g_{\sigma_j}}\Big).
\]
The first sum $-\sum_i g_{\sigma_i}$ is permutation-invariant; maximizing $\ell_{\PL}$ in $\sigma$ is equivalent to maximizing $F(\sigma)=\sum_{i=1}^K\log\sum_{j=i}^K e^{g_{\sigma_j}}$.

Suppose $\sigma$ has an adjacent inversion at position $t$: $a:=g_{\sigma_t}>b:=g_{\sigma_{t+1}}$. Set $R:=\sum_{j\ge t+2}e^{g_{\sigma_j}}\ge 0$. Only the $t$-th and $(t+1)$-st suffix terms can change under the swap of positions $t,t+1$: all later suffixes contain the same multiset and are unchanged, all earlier suffixes contain $\{a,b\}$ jointly and are also unchanged. The $t$-th term is $\log(e^a+e^b+R)$ before the swap and $\log(e^b+e^a+R)$ after the swap, hence equal. The $(t+1)$-st term changes from $\log(e^b+R)$ before the swap to $\log(e^a+R)$ after the swap. Since $a>b$ and $R\ge 0$, $\log(e^a+R)>\log(e^b+R)$, so $F$ strictly increases.

By repeated adjacent swaps that fix local inversions, every permutation can be transformed into the unique inversion-free permutation, the nondecreasing-score order. At each step $F$ strictly increases (modulo equal scores, where the swap leaves $F$ unchanged). Hence the ascending-score permutation attains the maximum. With ties, any consistent within-group ordering achieves the same maximum. Computing such a $\sigma_{\mathrm{worst}}$ requires only sorting $K$ scores, which costs $O(K\log K)$.
\end{proof}

\section{Offline theory: full proofs}
\label{app:offline}

\subsection{Convexity (\Cref{prop:offline-convex})}
\label{app:offline-convex}

\begin{proof}
Under \cref{ass:offline}, the induced PL score is affine in \(\theta\). In particular, up to a
prompt-dependent additive term that cancels in the PL loss, we may write
\[
s_\theta(x,y)=(\theta-\theta_{\rm ref})^\top\psi(x,y).
\]
For a fixed ranking \(\sigma\in\Sk\), the PL loss is
\[
\ell_{\PL}(\theta;x,Y,\sigma)
=
-\sum_{i=1}^K s_\theta(x,y_{\sigma_i})
+
\sum_{i=1}^K
\log\sum_{j=i}^K
\exp\big(s_\theta(x,y_{\sigma_j})\big).
\]
The first term is affine in \(\theta\). Each term in the second sum is a log-sum-exp of affine
functions of \(\theta\), and is therefore convex. Hence
\[
\theta\mapsto \ell_{\PL}(\theta;x,Y,\sigma)
\]
is convex for every fixed \((x,Y,\sigma)\).

By \Cref{lem:tv-decomp}, the robust loss admits the decomposition
\[
\ell_{\rob}(\theta;x,Y,\sigma^\star)
=
(1-\rho)\ell_{\PL}(\theta;x,Y,\sigma^\star)
+
\rho\max_{\sigma\in\Sk}
\ell_{\PL}(\theta;x,Y,\sigma).
\]
The first term is convex, and the second term is a pointwise maximum of convex functions, hence
convex. Since \(\rho\in[0,1]\), \(\ell_{\rob}(\cdot;x,Y,\sigma^\star)\) is convex. Finally, the
offline objective
\[
J_{\rob}(\theta)
=
\E_{\xi}[\ell_{\rob}(\theta;\xi)]
\]
is an expectation of convex functions and therefore convex on \(\Theta\).
\end{proof}

\subsection{Subgradient bound (used in \Cref{thm:offline-rate})}

\begin{lemma}[Bounded sample subgradient]
\label{lem:offline-subg-bound}
Suppose \cref{ass:offline} holds. Then for any sample \(\xi=(x,Y,\sigma^\star)\), any
\(\theta\in\Theta\), and any
\[
g\in\partial_\theta\ell_{\rob}(\theta;\xi),
\]
we have
\[
\|g\|_2\le 2KB_\psi.
\]
\end{lemma}

\begin{proof}
Fix \((x,Y,\sigma)\), where \(Y=(y_1,\ldots,y_K)\). Write
\[
s_i:=s_\theta(x,y_{\sigma_i}),
\qquad
\psi_i:=\psi(x,y_{\sigma_i}).
\]
The PL loss decomposes into stagewise terms:
\[
\ell_{\PL}(\theta;x,Y,\sigma)
=
\sum_{i=1}^K
\left[
-s_i+\log\sum_{j=i}^K\exp(s_j)
\right].
\]
For each stage \(i\), define the softmax weights
\[
p_{j\mid i}(\theta)
:=
\frac{\exp(s_j)}
{\sum_{m=i}^K\exp(s_m)},
\qquad j=i,\ldots,K.
\]
Differentiating the \(i\)-th stage gives
\[
\nabla_\theta \ell_i(\theta)
=
-\psi_i
+
\sum_{j=i}^K p_{j\mid i}(\theta)\psi_j.
\]
Since the weights \(p_{j\mid i}\) form a probability distribution and
\(\|\psi(x,y)\|\le B_\psi\), we have
\[
\left\|
\sum_{j=i}^K p_{j\mid i}(\theta)\psi_j
\right\|
\le
\sum_{j=i}^K p_{j\mid i}(\theta)\|\psi_j\|
\le
B_\psi.
\]
Therefore
\[
\|\nabla_\theta \ell_i(\theta)\|
\le
\|\psi_i\|
+
\left\|
\sum_{j=i}^K p_{j\mid i}(\theta)\psi_j
\right\|
\le
2B_\psi.
\]
Summing over \(i=1,\ldots,K\), we obtain
\[
\|\nabla_\theta\ell_{\PL}(\theta;x,Y,\sigma)\|
\le
\sum_{i=1}^K
\|\nabla_\theta\ell_i(\theta)\|
\le
2KB_\psi.
\]

Now consider the robust loss. By \Cref{lem:tv-decomp},
\[
\ell_{\rob}
=
(1-\rho)\ell_{\PL}(\sigma^\star)
+
\rho\max_{\sigma\in\Sk}\ell_{\PL}(\sigma).
\]
The subdifferential of the finite maximum is the convex hull of active PL gradients. Hence every
\(g\in\partial_\theta\ell_{\rob}(\theta;\xi)\) can be written as a convex combination of
\[
\nabla_\theta\ell_{\PL}(\theta;x,Y,\sigma^\star)
\]
and active gradients
\[
\nabla_\theta\ell_{\PL}(\theta;x,Y,\sigma),
\qquad
\sigma\in
\argmax_{\sigma'\in\Sk}
\ell_{\PL}(\theta;x,Y,\sigma').
\]
Each such PL gradient has norm at most \(2KB_\psi\). Since convex combinations preserve the same
norm bound, we conclude that
\[
\|g\|_2\le 2KB_\psi.
\]
\end{proof}

\subsection{Subgradient rate (\Cref{thm:offline-rate})}
\label{app:offline-rate}

\begin{proof}
Since \(\Theta\) is compact and \(J_{\rob}\) is convex and continuous, a minimizer exists. Let
\[
\theta^\star\in\argmin_{\theta\in\Theta}J_{\rob}(\theta).
\]
The projected stochastic subgradient update is
\[
\theta_{t+1}
=
\Pi_\Theta(\theta_t-\eta\widehat g_t).
\]
By non-expansiveness of the Euclidean projection onto the closed convex set \(\Theta\),
\[
\|\theta_{t+1}-\theta^\star\|^2
\le
\|\theta_t-\eta\widehat g_t-\theta^\star\|^2.
\]
Expanding the right-hand side gives
\[
\|\theta_{t+1}-\theta^\star\|^2
\le
\|\theta_t-\theta^\star\|^2
-
2\eta
\langle \widehat g_t,\theta_t-\theta^\star\rangle
+
\eta^2\|\widehat g_t\|^2.
\]
Taking conditional expectation given \(\theta_t\), and writing
\[
h_t:=\E[\widehat g_t\mid\theta_t]\in\partial J_{\rob}(\theta_t),
\]
we obtain
\[
\E[
\|\theta_{t+1}-\theta^\star\|^2
\mid
\theta_t
]
\le
\|\theta_t-\theta^\star\|^2
-
2\eta
\langle h_t,\theta_t-\theta^\star\rangle
+
\eta^2
\E[\|\widehat g_t\|^2\mid\theta_t].
\]
By convexity of \(J_{\rob}\),
\[
J_{\rob}(\theta_t)-J_{\rob}(\theta^\star)
\le
\langle h_t,\theta_t-\theta^\star\rangle.
\]
Using the assumed mini-batch second-moment bound
\[
\E[\|\widehat g_t\|^2\mid\theta_t]
\le
4K^2B_\psi^2+\frac{\sigma_g^2}{B_s},
\]
we get
\[
2\eta
\big(
J_{\rob}(\theta_t)-J_{\rob}(\theta^\star)
\big)
\le
\|\theta_t-\theta^\star\|^2
-
\E[
\|\theta_{t+1}-\theta^\star\|^2
\mid
\theta_t
]
+
\eta^2
\left(
4K^2B_\psi^2+\frac{\sigma_g^2}{B_s}
\right).
\]
Taking total expectation and summing over \(t=0,\ldots,T-1\), we obtain
\[
2\eta
\sum_{t=0}^{T-1}
\E[
J_{\rob}(\theta_t)-J_{\rob}(\theta^\star)
]
\le
\|\theta_0-\theta^\star\|^2
+
T\eta^2
\left(
4K^2B_\psi^2+\frac{\sigma_g^2}{B_s}
\right).
\]
Since \(\Theta=\{\theta:\|\theta\|\le B\}\), both \(\theta_0\) and \(\theta^\star\) belong to
\(\Theta\), and hence
\[
\|\theta_0-\theta^\star\|\le 2B.
\]
Therefore
\[
\frac1T
\sum_{t=0}^{T-1}
\E[
J_{\rob}(\theta_t)-J_{\rob}(\theta^\star)
]
\le
\frac{2B^2}{\eta T}
+
\frac{\eta}{2}
\left(
4K^2B_\psi^2+\frac{\sigma_g^2}{B_s}
\right).
\]
By convexity of \(J_{\rob}\) and Jensen's inequality, for
\[
\bar\theta_T
:=
\frac1T
\sum_{t=0}^{T-1}
\theta_t,
\]
we have
\[
\E[J_{\rob}(\bar\theta_T)]
\le
\frac1T
\sum_{t=0}^{T-1}
\E[J_{\rob}(\theta_t)].
\]
Hence
\[
\E[J_{\rob}(\bar\theta_T)]
-
J_{\rob}(\theta^\star)
\le
\frac{2B^2}{\eta T}
+
\frac{\eta}{2}
\left(
4K^2B_\psi^2+\frac{\sigma_g^2}{B_s}
\right).
\]
Choosing
\[
\eta
=
\frac{2B}
{\sqrt{
T\left(
4K^2B_\psi^2+\sigma_g^2/B_s
\right)
}}
\]
balances the two terms and gives
\[
\E[J_{\rob}(\bar\theta_T)]
-
\min_{\theta\in\Theta}J_{\rob}(\theta)
\le
\frac{
2B
\sqrt{
4K^2B_\psi^2+\sigma_g^2/B_s
}
}
{\sqrt T}.
\]
Solving for \(T\) to achieve \(\varepsilon\)-suboptimality gives
\[
T
=
O\left(
\frac{
4K^2B_\psi^2+\sigma_g^2/B_s
}{\varepsilon^2}
\right).
\]
With fixed \(B_s=\Theta(1)\) independent of \(\varepsilon\), the total sample complexity satisfies
\[
B_sT=O(\varepsilon^{-2}).
\]
\end{proof}

\section{Online theory: full proofs}
\label{app:online-proofs}

The proof strategy in this section follows the weakly-convex stochastic-subgradient framework~\citep{li2026oraclerobustonlinealignmentlarge, clarke1990optimization}. However, our listwise setting introduces an additional nonsmooth finite-max structure: the robust correction involves \(\max_{\sigma\in\mathfrak S_K}\ell_{\PL}(\theta;x,Y,\sigma)\). Consequently, beyond the standard score-function correction for policy-induced sampling, we must explicitly control the Clarke subdifferential of this finite maximum and fix a measurable tie-breaking rule for nonunique worst-case rankings.

We use the notation of \Cref{ass:online}: $\psi$ has $\|\psi\|\le B_\psi$, $\Theta$ is closed, convex, bounded with diameter $D$, $\pi_\theta\propto\exp(\theta^\top\psi)$. Recall $s_\theta(x,y):=\log[\pi_\theta/\pi_{\mathrm{ref}}]=(\theta-\theta_{\mathrm{ref}})^\top\psi(x,y)+\mathrm{const}$ and $S_\theta(x,Y):=\nabla_\theta\log P_{\pi_\theta}(Y\mid x)$.

\subsection{Score-function bounds and PL bounds}

\begin{lemma}[Score-function bounds]
\label{lem:score-bounds}
Suppose \cref{ass:offline,ass:online} hold,
\[
\nabla_\theta\log\pi_\theta(y\mid x)
=
\psi(x,y)-\bar\psi_\theta(x),
\qquad
\bar\psi_\theta(x)
:=
\E_{y'\sim\pi_\theta(\cdot\mid x)}
[
\psi(x,y')
].
\]
Consequently,
\[
\|\nabla_\theta\log\pi_\theta(y\mid x)\|
\le
2B_\psi,
\qquad
\|S_\theta(x,Y)\|
\le
2KB_\psi,
\qquad
\|\nabla_\theta S_\theta(x,Y)\|_\mathrm{op}
\le
KB_\psi^2.
\]
\end{lemma}

\begin{proof}
Fix \(x\) and write
\[
A_\theta(x)
:=
\log\sum_{y'\in\mathcal Y}
\exp(\theta^\top\psi(x,y')).
\]
Under the log-linear policy class,
\[
\log\pi_\theta(y\mid x)
=
\theta^\top\psi(x,y)-A_\theta(x).
\]
Differentiating gives
\[
\nabla_\theta\log\pi_\theta(y\mid x)
=
\psi(x,y)-\nabla_\theta A_\theta(x).
\]
Moreover,
\[
\nabla_\theta A_\theta(x)
=
\frac{\sum_{y'\in\mathcal Y}
\exp(\theta^\top\psi(x,y'))\psi(x,y')}
{\sum_{y'\in\mathcal Y}
\exp(\theta^\top\psi(x,y'))}
=
\E_{y'\sim\pi_\theta(\cdot\mid x)}
[
\psi(x,y')
]
=
\bar\psi_\theta(x).
\]
Therefore,
\[
\nabla_\theta\log\pi_\theta(y\mid x)
=
\psi(x,y)-\bar\psi_\theta(x).
\]

Since \(\|\psi(x,y)\|\le B_\psi\) for every \((x,y)\), Jensen's inequality gives
\[
\|\bar\psi_\theta(x)\|
=
\left\|
\E_{y'\sim\pi_\theta(\cdot\mid x)}
[
\psi(x,y')
]
\right\|
\le
\E_{y'\sim\pi_\theta(\cdot\mid x)}
[
\|\psi(x,y')\|
]
\le
B_\psi.
\]
Hence
\[
\|\nabla_\theta\log\pi_\theta(y\mid x)\|
\le
\|\psi(x,y)\|+\|\bar\psi_\theta(x)\|
\le
2B_\psi.
\]

Now let \(Y=(y_1,\ldots,y_K)\). Since the list is sampled iid from
\(\pi_\theta(\cdot\mid x)\),
\[
P_{\pi_\theta}(Y\mid x)
=
\prod_{i=1}^K\pi_\theta(y_i\mid x),
\]
and therefore
\[
S_\theta(x,Y)
:=
\nabla_\theta\log P_{\pi_\theta}(Y\mid x)
=
\sum_{i=1}^K
\nabla_\theta\log\pi_\theta(y_i\mid x).
\]
Using the bound above,
\[
\|S_\theta(x,Y)\|
\le
\sum_{i=1}^K
\|\nabla_\theta\log\pi_\theta(y_i\mid x)\|
\le
2KB_\psi.
\]

It remains to bound \(\nabla_\theta S_\theta(x,Y)\). Since
\[
\nabla_\theta\log\pi_\theta(y\mid x)
=
\psi(x,y)-\bar\psi_\theta(x),
\]
and \(\psi(x,y)\) is independent of \(\theta\),
\[
\nabla_\theta^2\log\pi_\theta(y\mid x)
=
-\nabla_\theta\bar\psi_\theta(x).
\]
We compute \(\nabla_\theta\bar\psi_\theta(x)\). For any vector \(v\in\mathbb R^d\),
\[
\nabla_\theta
\E_{y'\sim\pi_\theta(\cdot\mid x)}
[
\psi(x,y')
]
=
\E_{y'\sim\pi_\theta(\cdot\mid x)}
[
(\psi(x,y')-\bar\psi_\theta(x))\psi(x,y')^\top
],
\]
which is the covariance matrix
\[
\operatorname{Cov}_{y'\sim\pi_\theta(\cdot\mid x)}
[
\psi(x,y')
]
=
\E[
(\psi-\bar\psi_\theta)(\psi-\bar\psi_\theta)^\top
].
\]
Thus
\[
\nabla_\theta^2\log\pi_\theta(y\mid x)
=
-
\operatorname{Cov}_{y'\sim\pi_\theta(\cdot\mid x)}
[
\psi(x,y')
].
\]
Consequently,
\[
\nabla_\theta S_\theta(x,Y)
=
\sum_{i=1}^K
\nabla_\theta^2\log\pi_\theta(y_i\mid x)
=
-
K\,
\operatorname{Cov}_{y'\sim\pi_\theta(\cdot\mid x)}
[
\psi(x,y')
],
\]
because the covariance term depends on \(x\) and \(\theta\), but not on the particular sampled
response \(y_i\).

Finally, for any unit vector \(v\),
\[
v^\top
\operatorname{Cov}_{y'\sim\pi_\theta(\cdot\mid x)}
[
\psi(x,y')
]
v
=
\operatorname{Var}_{y'\sim\pi_\theta(\cdot\mid x)}
(
v^\top\psi(x,y')
)
\le
\E[
(v^\top\psi(x,y'))^2
]
\le
B_\psi^2.
\]
Therefore,
\[
\left\|
\operatorname{Cov}_{y'\sim\pi_\theta(\cdot\mid x)}
[
\psi(x,y')
]
\right\|_\mathrm{op}
\le
B_\psi^2.
\]
Hence
\[
\|\nabla_\theta S_\theta(x,Y)\|_\mathrm{op}
\le
K B_\psi^2.
\]
\end{proof}

\begin{lemma}[Convexity, magnitude, curvature of \(\ell_{\PL}\)]
\label{lem:PL-bounds}
Suppose \cref{ass:offline,ass:online} hold, for every fixed \((x,Y,\sigma)\), the map
\(\theta\mapsto\ell_{\PL}(\theta;x,Y,\sigma)\) is convex and hence Clarke regular. Moreover, on
\(\Theta\),
\[
0\le\ell_{\PL}(\theta;x,Y,\sigma)\le C_L,
\qquad
\|\nabla_\theta\ell_{\PL}(\theta;x,Y,\sigma)\|\le C_G,
\qquad
0\preceq \nabla_\theta^2\ell_{\PL}(\theta;x,Y,\sigma)\preceq C_H I,
\]
where
\[
C_L = K(\log K+2DB_\psi),
\qquad
C_G = 2KB_\psi,
\qquad
C_H = KB_\psi^2.
\]
\end{lemma}

\begin{proof}
Fix \((x,Y,\sigma)\), where \(Y=(y_1,\ldots,y_K)\). Recall that the induced affine score can be
written as
\[
s_\theta(x,y)
=
(\theta-\theta_{\rm ref})^\top\psi(x,y),
\]
up to an additive term depending only on \(x\), which cancels in the PL loss. For notational
simplicity, write
\[
s_i:=s_\theta(x,y_{\sigma_i}),
\qquad
\psi_i:=\psi(x,y_{\sigma_i}).
\]
The PL loss decomposes into \(K\) stagewise losses:
\[
\ell_{\PL}(\theta;x,Y,\sigma)
=
\sum_{i=1}^K \ell_i(\theta),
\]
where
\[
\ell_i(\theta)
:=
-s_i+\log\sum_{j=i}^K \exp(s_j).
\]

\paragraph{Convexity.}
For each \(i\), the term \(-s_i\) is affine in \(\theta\), and
\[
\theta\mapsto \log\sum_{j=i}^K \exp(s_\theta(x,y_{\sigma_j}))
\]
is a log-sum-exp of affine functions, hence convex. Therefore each \(\ell_i\) is convex, and so
\(\ell_{\PL}=\sum_{i=1}^K\ell_i\) is convex. Since it is finite-valued and convex, it is Clarke
regular.

\paragraph{Magnitude bound.}
For each stage \(i\),
\[
\ell_i(\theta)
=
\log\sum_{j=i}^K \exp(s_j-s_i).
\]
Since \(\theta\in\Theta\) and
\[
D:=\sup_{\theta\in\Theta}\|\theta-\theta_{\rm ref}\|,
\]
we have, for any \(j\),
\[
|s_j-s_i|
=
|(\theta-\theta_{\rm ref})^\top(\psi_j-\psi_i)|
\le
\|\theta-\theta_{\rm ref}\|\,\|\psi_j-\psi_i\|
\le
2DB_\psi.
\]
Also, the term \(j=i\) equals \(\exp(s_i-s_i)=1\), so
\[
\ell_i(\theta)
=
\log\sum_{j=i}^K \exp(s_j-s_i)
\ge
0.
\]
For the upper bound,
\[
\ell_i(\theta)
\le
\log\sum_{j=i}^K \exp(2DB_\psi)
\le
\log K+2DB_\psi.
\]
Summing over \(i=1,\ldots,K\) gives
\[
0\le
\ell_{\PL}(\theta;x,Y,\sigma)
\le
K(\log K+2DB_\psi)
=
C_L.
\]

\paragraph{Gradient bound.}
For each stage \(i\), define the stagewise softmax weights
\[
p_{j|i}(\theta)
:=
\frac{\exp(s_j)}
{\sum_{m=i}^K\exp(s_m)},
\qquad j=i,\ldots,K.
\]
Then
\[
\nabla_\theta \ell_i(\theta)
=
-\psi_i+\sum_{j=i}^K p_{j|i}(\theta)\psi_j.
\]
Since the weights \(p_{j|i}\) form a probability distribution over \(\{i,\ldots,K\}\),
\[
\left\|
\sum_{j=i}^K p_{j|i}(\theta)\psi_j
\right\|
\le
\sum_{j=i}^K p_{j|i}(\theta)\|\psi_j\|
\le
B_\psi.
\]
Thus
\[
\|\nabla_\theta \ell_i(\theta)\|
\le
\|\psi_i\|
+
\left\|
\sum_{j=i}^K p_{j|i}(\theta)\psi_j
\right\|
\le
2B_\psi.
\]
Summing over \(i\) yields
\[
\|\nabla_\theta\ell_{\PL}(\theta;x,Y,\sigma)\|
\le
\sum_{i=1}^K
\|\nabla_\theta\ell_i(\theta)\|
\le
2KB_\psi
=
C_G.
\]

\paragraph{Hessian bound.}
For each stage \(i\), differentiating the stagewise softmax gradient gives
\[
\nabla_\theta^2\ell_i(\theta)
=
\sum_{j=i}^K p_{j|i}(\theta)
(\psi_j-\bar\psi_i)(\psi_j-\bar\psi_i)^\top,
\]
where
\[
\bar\psi_i
:=
\sum_{j=i}^K p_{j|i}(\theta)\psi_j.
\]
Equivalently,
\[
\nabla_\theta^2\ell_i(\theta)
=
\operatorname{Cov}_{j\sim p_{\cdot|i}}
[
\psi_j
].
\]
Therefore
\[
\nabla_\theta^2\ell_i(\theta)\succeq 0.
\]
Moreover, for any unit vector \(v\),
\[
v^\top\nabla_\theta^2\ell_i(\theta)v
=
\operatorname{Var}_{j\sim p_{\cdot|i}}(v^\top\psi_j)
\le
\E_{j\sim p_{\cdot|i}}[(v^\top\psi_j)^2]
\le
B_\psi^2.
\]
Hence
\[
0\preceq \nabla_\theta^2\ell_i(\theta)\preceq B_\psi^2 I.
\]
Summing over \(i=1,\ldots,K\), we obtain
\[
0
\preceq
\nabla_\theta^2\ell_{\PL}(\theta;x,Y,\sigma)
=
\sum_{i=1}^K\nabla_\theta^2\ell_i(\theta)
\preceq
KB_\psi^2 I
=
C_H I.
\]
This completes the proof.
\end{proof}

\begin{corollary}[Convexity of fixed-list robust loss]
\label{cor:fixedlist-convex}
$\theta\mapsto\ell_{\rob}(\theta;x,Y,\sigma^\star)$ is convex. Defining $L_{\rob}(\theta;x,Y) := \E_{\sigma^\star\sim p^\star(\cdot\mid x,Y)}[\ell_{\rob}(\theta;x,Y,\sigma^\star)]$, $L_{\rob}$ is convex with $0\le L_{\rob}\le C_L$ and $\|\partial_C L_{\rob}\|\le C_G$.
\end{corollary}

\subsection{Smoothing and weak convexity (\Cref{prop:online-wc})}

\begin{definition}[Smoothed objects]
\label{def:smoothed}
For $\tau>0$, let $M_\tau(\theta;x,Y):=\tau\log\sum_{\sigma\in\Sk}\exp(\ell_{\PL}(\theta;x,Y,\sigma)/\tau)$,
$\ell_{\rob}^\tau:=(1-\rho)\ell_{\PL}(\sigma^\star)+\rho M_\tau$, and
$J_{\rob}^{\mathrm{on},\tau}(\theta) := \E_{x,Y,\sigma^\star}[\ell_{\rob}^\tau]$.
\end{definition}

\begin{lemma}[Properties of \(M_\tau\)]
\label{lem:Mtau}
For \(\tau>0\), define
\[
M_\tau(\theta;x,Y)
:=
\tau\log
\sum_{\sigma\in\Sk}
\exp\left(
\frac{\ell_{\PL}(\theta;x,Y,\sigma)}{\tau}
\right).
\]
Then:
\textup{(i)}
\[
\max_{\sigma\in\Sk}\ell_{\PL}(\theta;x,Y,\sigma)
\le
M_\tau(\theta;x,Y)
\le
\max_{\sigma\in\Sk}\ell_{\PL}(\theta;x,Y,\sigma)
+
\tau\log K!.
\]
\textup{(ii)} \(\theta\mapsto M_\tau(\theta;x,Y)\) is convex and \(C^\infty\).

\textup{(iii)}
\[
\|\nabla_\theta M_\tau(\theta;x,Y)\|\le C_G,
\qquad
0\preceq
\nabla_\theta^2 M_\tau(\theta;x,Y)
\preceq
\left(C_H+\frac{C_G^2}{\tau}\right)I.
\]
\end{lemma}

\begin{proof}
Fix \((x,Y)\) and write
\[
\ell_\sigma(\theta)
:=
\ell_{\PL}(\theta;x,Y,\sigma),
\qquad
\sigma\in\Sk.
\]
Then
\[
M_\tau(\theta;x,Y)
=
\tau\log
\sum_{\sigma\in\Sk}
\exp\left(
\frac{\ell_\sigma(\theta)}{\tau}
\right).
\]

\textup{(i)}
Let
\[
m(\theta)
:=
\max_{\sigma\in\Sk}\ell_\sigma(\theta).
\]
Then
\[
\sum_{\sigma\in\Sk}
\exp\left(
\frac{\ell_\sigma(\theta)}{\tau}
\right)
=
\exp\left(\frac{m(\theta)}{\tau}\right)
\sum_{\sigma\in\Sk}
\exp\left(
\frac{\ell_\sigma(\theta)-m(\theta)}{\tau}
\right).
\]
Since \(\ell_\sigma(\theta)-m(\theta)\le 0\) for every \(\sigma\), and at least one maximizer attains
zero, we have
\[
1
\le
\sum_{\sigma\in\Sk}
\exp\left(
\frac{\ell_\sigma(\theta)-m(\theta)}{\tau}
\right)
\le
|\Sk|
=
K!.
\]
Taking \(\tau\log(\cdot)\) gives
\[
m(\theta)
\le
M_\tau(\theta;x,Y)
\le
m(\theta)+\tau\log K!,
\]
which proves \textup{(i)}.

\textup{(ii)}
By \Cref{lem:PL-bounds}, for every \(\sigma\in\Sk\), the map
\[
\theta\mapsto \ell_\sigma(\theta)
\]
is convex and \(C^\infty\). Since \(\Sk\) is finite, \(M_\tau\) is a finite log-sum-exp composition of
\(C^\infty\) functions, and hence is \(C^\infty\).

To prove convexity, take any \(\theta_1,\theta_2\) and \(\alpha\in[0,1]\). Since each
\(\ell_\sigma\) is convex,
\[
\ell_\sigma(\alpha\theta_1+(1-\alpha)\theta_2)
\le
\alpha \ell_\sigma(\theta_1)+(1-\alpha)\ell_\sigma(\theta_2).
\]
Therefore,
\[
\exp\left(
\frac{\ell_\sigma(\alpha\theta_1+(1-\alpha)\theta_2)}{\tau}
\right)
\le
\exp\left(
\frac{\alpha \ell_\sigma(\theta_1)}{\tau}
\right)
\exp\left(
\frac{(1-\alpha)\ell_\sigma(\theta_2)}{\tau}
\right).
\]
Summing over \(\sigma\) and applying Hölder's inequality gives
\[
\sum_{\sigma\in\Sk}
\exp\left(
\frac{\ell_\sigma(\alpha\theta_1+(1-\alpha)\theta_2)}{\tau}
\right)
\le
\left[
\sum_{\sigma\in\Sk}
\exp\left(
\frac{\ell_\sigma(\theta_1)}{\tau}
\right)
\right]^\alpha
\left[
\sum_{\sigma\in\Sk}
\exp\left(
\frac{\ell_\sigma(\theta_2)}{\tau}
\right)
\right]^{1-\alpha}.
\]
Taking \(\tau\log(\cdot)\) yields
\[
M_\tau(\alpha\theta_1+(1-\alpha)\theta_2)
\le
\alpha M_\tau(\theta_1)
+
(1-\alpha)M_\tau(\theta_2).
\]
Thus \(M_\tau\) is convex.

\textup{(iii)}
Define the softmax weights over rankings
\[
w_\sigma(\theta)
:=
\frac{
\exp(\ell_\sigma(\theta)/\tau)
}{
\sum_{\sigma'\in\Sk}
\exp(\ell_{\sigma'}(\theta)/\tau)
}.
\]
Then \(w_\sigma(\theta)\ge0\) and \(\sum_{\sigma\in\Sk}w_\sigma(\theta)=1\). Differentiating
\(M_\tau\) gives
\[
\nabla_\theta M_\tau(\theta;x,Y)
=
\sum_{\sigma\in\Sk}
w_\sigma(\theta)
\nabla_\theta\ell_\sigma(\theta).
\]
By \Cref{lem:PL-bounds}, \(\|\nabla_\theta\ell_\sigma(\theta)\|\le C_G\) for every
\(\sigma\). Hence
\[
\|\nabla_\theta M_\tau(\theta;x,Y)\|
\le
\sum_{\sigma\in\Sk}
w_\sigma(\theta)
\|\nabla_\theta\ell_\sigma(\theta)\|
\le
C_G.
\]

For the Hessian, differentiating the gradient gives
\[
\nabla_\theta^2 M_\tau(\theta;x,Y)
=
\sum_{\sigma\in\Sk}
w_\sigma(\theta)
\nabla_\theta^2\ell_\sigma(\theta)
+
\frac{1}{\tau}
\left[
\sum_{\sigma\in\Sk}
w_\sigma(\theta)
\nabla_\theta\ell_\sigma(\theta)
\nabla_\theta\ell_\sigma(\theta)^\top
-
\nabla_\theta M_\tau(\theta;x,Y)
\nabla_\theta M_\tau(\theta;x,Y)^\top
\right].
\]
The second bracket is the covariance matrix of the random vector
\(\nabla_\theta\ell_\sigma(\theta)\) under \(\sigma\sim w(\theta)\), and is therefore positive
semidefinite. Since each \(\nabla_\theta^2\ell_\sigma(\theta)\succeq0\), we get
\[
\nabla_\theta^2 M_\tau(\theta;x,Y)\succeq0.
\]

For the upper bound, by \Cref{lem:PL-bounds},
\[
\sum_{\sigma\in\Sk}
w_\sigma(\theta)
\nabla_\theta^2\ell_\sigma(\theta)
\preceq
C_H I.
\]
For the covariance term, for any unit vector \(v\),
\[
v^\top
\left[
\sum_{\sigma}
w_\sigma
\nabla\ell_\sigma\nabla\ell_\sigma^\top
-
\nabla M_\tau\nabla M_\tau^\top
\right]
v
=
\operatorname{Var}_{\sigma\sim w(\theta)}
\left(
v^\top\nabla_\theta\ell_\sigma(\theta)
\right)
\le
\E_{\sigma\sim w(\theta)}
\left[
(v^\top\nabla_\theta\ell_\sigma(\theta))^2
\right].
\]
Using \(\|\nabla_\theta\ell_\sigma(\theta)\|\le C_G\), we have
\[
\E_{\sigma\sim w(\theta)}
\left[
(v^\top\nabla_\theta\ell_\sigma(\theta))^2
\right]
\le
C_G^2.
\]
Therefore the covariance term is bounded by \(C_G^2 I\), and hence
\[
\nabla_\theta^2M_\tau(\theta;x,Y)
\preceq
\left(C_H+\frac{C_G^2}{\tau}\right)I.
\]
This proves \textup{(iii)}.
\end{proof}

\begin{proposition}[Smoothed online weak convexity]
\label{prop:smoothed-wc}
Suppose \cref{ass:offline,ass:online} hold, $J_{\rob}^{\mathrm{on},\tau}\in C^2$ an open neighborhood of \(\Theta\) and is $\kappa^\tau$-weakly convex with
$\kappa^\tau \le 8K^2B_\psi^2 + (C_L+\rho\tau\log K!)\cdot KB_\psi^2$.
\end{proposition}

\begin{proof}
Fix \(\tau>0\). Recall that the smoothed online objective is
\[
J_{\rob}^{\mathrm{on},\tau}(\theta)
=
\E_{x\sim\mathcal D_x}
\sum_{Y\in\mathcal Y^K}
P_{\pi_\theta}(Y\mid x)
L_\tau(\theta;x,Y),
\]
where
\[
L_\tau(\theta;x,Y)
:=
\E_{\sigma^\star\sim p^\star(\cdot\mid x,Y)}
[
\ell_{\rob}^{\tau}(\theta;x,Y,\sigma^\star)
].
\]
Since \(\mathcal Y\) and \(\Sk\) are finite, all sums over \(Y\) and \(\sigma^\star\) are finite.
Moreover,
\[
\ell_{\rob}^{\tau}(\theta;x,Y,\sigma^\star)
=
(1-\rho)\ell_{\PL}(\theta;x,Y,\sigma^\star)
+
\rho M_\tau(\theta;x,Y),
\]
with
\[
M_\tau(\theta;x,Y)
:=
\tau\log
\sum_{\sigma\in\Sk}
\exp\left(
\frac{\ell_{\PL}(\theta;x,Y,\sigma)}{\tau}
\right).
\]
By \Cref{lem:PL-bounds}, each
\(\ell_{\PL}(\cdot;x,Y,\sigma)\) is convex and \(C^\infty\) in \(\theta\).
The log-sum-exp smoothing \(M_\tau\) is therefore convex and \(C^\infty\). Hence
\(\ell_{\rob}^{\tau}\) and \(L_\tau\) are convex and \(C^\infty\) in \(\theta\).

For each fixed \(x\), define
\[
H_x(\theta)
:=
\sum_{Y\in\mathcal Y^K}
P_{\pi_\theta}(Y\mid x)L_\tau(\theta;x,Y).
\]
Because the log-linear softmax policy is \(C^\infty\) and \(\mathcal Y^K\) is finite,
\(H_x\) is \(C^2\). The uniform bounds from \Cref{lem:PL-bounds},
\Cref{lem:Mtau}, and \Cref{lem:score-bounds} imply that the first and second derivatives of
\(H_x\) are dominated uniformly over \(x\) and \(\theta\in\Theta\). Therefore differentiation may
be interchanged with the expectation over \(x\), and
\[
J_{\rob}^{\mathrm{on},\tau}(\theta)=\E_x[H_x(\theta)]
\]
is \(C^2\) on \(\Theta\).

We now lower bound the Hessian. For \(Y=(y_1,\ldots,y_K)\), write
\[
P_\theta(Y\mid x)
:=
P_{\pi_\theta}(Y\mid x)
=
\prod_{i=1}^K \pi_\theta(y_i\mid x),
\]
and define the list score function
\[
S_\theta(x,Y)
:=
\nabla_\theta\log P_\theta(Y\mid x)
=
\sum_{i=1}^K\nabla_\theta\log\pi_\theta(y_i\mid x).
\]
Then
\[
\nabla_\theta P_\theta(Y\mid x)
=
P_\theta(Y\mid x)S_\theta(x,Y).
\]
Differentiating once more gives
\[
\nabla_\theta^2P_\theta(Y\mid x)
=
\nabla_\theta\!\left(P_\theta(Y\mid x)S_\theta(x,Y)\right)
=
P_\theta(Y\mid x)
\left[
S_\theta(x,Y)S_\theta(x,Y)^\top
+
\nabla_\theta S_\theta(x,Y)
\right].
\]

Applying the product rule twice to
\[
H_x(\theta)
=
\sum_{Y\in\mathcal Y^K}
P_\theta(Y\mid x)L_\tau(\theta;x,Y),
\]
we obtain
\begin{align}
\nabla_\theta^2 H_x(\theta)
&=
\sum_{Y\in\mathcal Y^K}
P_\theta(Y\mid x)
\bigg[
\underbrace{\nabla_\theta^2L_\tau(\theta;x,Y)}_{T_1}
\notag\\
&\qquad
+
\underbrace{
S_\theta(x,Y)\nabla_\theta L_\tau(\theta;x,Y)^\top
+
\nabla_\theta L_\tau(\theta;x,Y)S_\theta(x,Y)^\top
}_{T_2}
\notag\\
&\qquad
+
\underbrace{
L_\tau(\theta;x,Y)
\left(
S_\theta(x,Y)S_\theta(x,Y)^\top
+
\nabla_\theta S_\theta(x,Y)
\right)
}_{T_3}
\bigg].
\label{eq:hessian-Hx}
\end{align}

We bound the three terms from below. Since \(L_\tau(\cdot;x,Y)\) is convex,
\[
T_1=\nabla_\theta^2L_\tau(\theta;x,Y)\succeq 0.
\]
For the cross term \(T_2\), for any unit vector \(v\),
\[
v^\top T_2 v
=
2\langle v,S_\theta(x,Y)\rangle
\langle v,\nabla_\theta L_\tau(\theta;x,Y)\rangle.
\]
Therefore, by Cauchy--Schwarz,
\[
v^\top T_2 v
\ge
-2\|S_\theta(x,Y)\|
\|\nabla_\theta L_\tau(\theta;x,Y)\|.
\]
Equivalently,
\[
T_2
\succeq
-2\|S_\theta(x,Y)\|
\|\nabla_\theta L_\tau(\theta;x,Y)\|I.
\]

For \(T_3\), note that \(L_\tau(\theta;x,Y)\ge 0\) and
\[
S_\theta(x,Y)S_\theta(x,Y)^\top\succeq 0.
\]
Hence
\[
T_3
\succeq
L_\tau(\theta;x,Y)\nabla_\theta S_\theta(x,Y).
\]
By \Cref{lem:score-bounds},
\[
\|\nabla_\theta S_\theta(x,Y)\|_\mathrm{op}\le KB_\psi^2,
\]
and thus
\[
\nabla_\theta S_\theta(x,Y)
\succeq
-KB_\psi^2 I.
\]
Therefore,
\[
T_3
\succeq
-L_\tau(\theta;x,Y)KB_\psi^2I.
\]

Now use the uniform bounds
\[
\|S_\theta(x,Y)\|\le 2KB_\psi,
\qquad
\|\nabla_\theta L_\tau(\theta;x,Y)\|\le C_G=2KB_\psi,
\]
and
\[
0\le L_\tau(\theta;x,Y)\le C_L+\rho\tau\log K!.
\]
The gradient bound follows because \(\nabla M_\tau\) is a softmax-weighted convex combination of
the gradients \(\nabla\ell_{\PL}(\theta;x,Y,\sigma)\), each of which has norm at most \(C_G\).
The loss bound follows from
\[
M_\tau(\theta;x,Y)
\le
\max_{\sigma\in\Sk}\ell_{\PL}(\theta;x,Y,\sigma)
+
\tau\log K!,
\]
together with the uniform PL loss bound.

Combining the bounds, for every unit vector \(v\),
\begin{align}
v^\top \nabla_\theta^2H_x(\theta)v
&\ge
-2(2KB_\psi)(2KB_\psi)
-
(C_L+\rho\tau\log K!)KB_\psi^2
\notag\\
&=
-8K^2B_\psi^2
-
(C_L+\rho\tau\log K!)KB_\psi^2.
\end{align}
Since the bound is uniform in \(Y\) and \(x\), summing over \(Y\) with weights
\(P_\theta(Y\mid x)\) and taking expectation over \(x\) preserve the same lower bound. Hence
\[
\nabla_\theta^2J_{\rob}^{\mathrm{on},\tau}(\theta)
\succeq
-
\left[
8K^2B_\psi^2
+
(C_L+\rho\tau\log K!)KB_\psi^2
\right]I.
\]
Therefore \(J_{\rob}^{\mathrm{on},\tau}\) is \(\kappa^\tau\)-weakly convex with
\[
\kappa^\tau
\le
8K^2B_\psi^2
+
(C_L+\rho\tau\log K!)KB_\psi^2.
\]
Indeed, for a \(C^2\) function, the Hessian lower bound
\(\nabla^2 f\succeq -\kappa I\) is equivalent to convexity of
\(f+\frac{\kappa}{2}\|\cdot\|^2\), which is precisely \(\kappa\)-weak convexity.
\end{proof}

\begin{proof}[Proof of \Cref{prop:online-wc}]
For \(\tau>0\), define
\[
\bar\kappa_\tau
:=
8K^2B_\psi^2
+
(C_L+\rho\tau\log K!)KB_\psi^2.
\]
By \Cref{prop:smoothed-wc}, \(J_{\rob}^{\mathrm{on},\tau}\) is
\(\bar\kappa_\tau\)-weakly convex on \(\Theta\). Equivalently, the function
\[
\Phi_\tau(\theta)
:=
J_{\rob}^{\mathrm{on},\tau}(\theta)
+
\frac{\bar\kappa_\tau}{2}\|\theta\|^2
\]
is convex on \(\Theta\).

We first verify convergence of the smoothed objectives. By the log-sum-exp bound in
\Cref{lem:Mtau}(i),
\[
0
\le
M_\tau(\theta;x,Y)-M(\theta;x,Y)
\le
\tau\log K!.
\]
Since
\[
\ell_{\rob}^{\tau}
=
(1-\rho)\ell_{\PL}(\sigma^\star)
+
\rho M_\tau,
\qquad
\ell_{\rob}
=
(1-\rho)\ell_{\PL}(\sigma^\star)
+
\rho M,
\]
we have, for every \((\theta,x,Y,\sigma^\star)\),
\[
0
\le
\ell_{\rob}^{\tau}(\theta;x,Y,\sigma^\star)
-
\ell_{\rob}(\theta;x,Y,\sigma^\star)
\le
\rho\tau\log K!.
\]
Taking expectation over \(x\), \(Y\sim\pi_\theta^{\otimes K}(\cdot|x)\), and
\(\sigma^\star\sim p^\star(\cdot|x,Y)\) gives the uniform bound
\[
0
\le
J_{\rob}^{\mathrm{on},\tau}(\theta)
-
J_{\rob}^{\mathrm{on}}(\theta)
\le
\rho\tau\log K!,
\qquad
\forall \theta\in\Theta.
\]
Hence \(J_{\rob}^{\mathrm{on},\tau}\to J_{\rob}^{\mathrm{on}}\) uniformly on \(\Theta\).

Now define
\[
\kappa
:=
\lim_{\tau\downarrow0}\bar\kappa_\tau
=
8K^2B_\psi^2+C_LKB_\psi^2.
\]
Since \(C_L=K(\log K+2DB_\psi)\), this becomes
\[
\kappa
=
K^2B_\psi^2(8+\log K+2DB_\psi).
\]

It remains to prove that
\[
\Phi(\theta)
:=
J_{\rob}^{\mathrm{on}}(\theta)
+
\frac{\kappa}{2}\|\theta\|^2
\]
is convex on \(\Theta\). Let \(\theta_1,\theta_2\in\Theta\) and
\(\alpha\in[0,1]\). Since \(\Theta\) is convex,
\[
\theta_\alpha
:=
\alpha\theta_1+(1-\alpha)\theta_2
\in\Theta.
\]
For every \(\tau>0\), convexity of \(\Phi_\tau\) gives
\[
\Phi_\tau(\theta_\alpha)
\le
\alpha\Phi_\tau(\theta_1)
+
(1-\alpha)\Phi_\tau(\theta_2).
\]
Taking \(\tau\downarrow0\), using the uniform convergence
\(J_{\rob}^{\mathrm{on},\tau}\to J_{\rob}^{\mathrm{on}}\) and
\(\bar\kappa_\tau\to\kappa\), yields
\[
\Phi(\theta_\alpha)
\le
\alpha\Phi(\theta_1)
+
(1-\alpha)\Phi(\theta_2).
\]
Therefore \(\Phi\) is convex on \(\Theta\). Equivalently,
\(J_{\rob}^{\mathrm{on}}\) is \(\kappa\)-weakly convex on \(\Theta\), with
\[
\kappa
=
K^2B_\psi^2(8+\log K+2DB_\psi).
\]
\end{proof}

\subsection{Constrained Moreau envelope}

For $F=J_{\rob}^{\mathrm{on}}+I_\Theta$ and $\hat\lambda\in(0,1/\kappa)$, define
$F_{\hat\lambda}(\theta)=\min_{u\in\Theta}\{J_{\rob}^{\mathrm{on}}(u)+\tfrac1{2\hat\lambda}\|u-\theta\|^2\}$ and $\hat\theta(\theta)=\prox_{\hat\lambda F}(\theta)\in\Theta$.

\begin{lemma}[Moreau-envelope Properties]
\label{lem:moreau}
Let
\[
F(\theta)
:=
J_{\rob}^{\mathrm{on}}(\theta)+I_\Theta(\theta),
\]
where \(I_\Theta\) is the indicator of the closed convex set \(\Theta\). Suppose
\(J_{\rob}^{\mathrm{on}}\) is \(\kappa\)-weakly convex on \(\Theta\), and assume that \(F\) is proper,
lower semicontinuous, and bounded below. Fix
\[
\hat\lambda\in(0,1/\kappa).
\]
Define the constrained Moreau envelope
\[
F_{\hat\lambda}(\theta)
:=
\min_{u\in\mathbb R^d}
\left\{
F(u)+\frac{1}{2\hat\lambda}\|u-\theta\|^2
\right\}
=
\min_{u\in\Theta}
\left\{
J_{\rob}^{\mathrm{on}}(u)+\frac{1}{2\hat\lambda}\|u-\theta\|^2
\right\},
\]
and denote the proximal point by
\[
\hat\theta(\theta)
:=
\prox_{\hat\lambda F}(\theta)
:=
\argmin_{u\in\mathbb R^d}
\left\{
F(u)+\frac{1}{2\hat\lambda}\|u-\theta\|^2
\right\}.
\]
Then the following hold:
\begin{enumerate}
    \item \(\prox_{\hat\lambda F}\) is single-valued on \(\mathbb R^d\).
    \item \(F_{\hat\lambda}\in C^1\), and
    \[
    \nabla F_{\hat\lambda}(\theta)
    =
    \hat\lambda^{-1}
    \left(
    \theta-\hat\theta(\theta)
    \right).
    \]
    \item \(\nabla F_{\hat\lambda}\) is Lipschitz with constant at most
    \[
    L_{\mathrm{env}}
    :=
    \frac{1}{\hat\lambda(1-\kappa\hat\lambda)}.
    \]
    \item \emph{Near-stationarity:}
    \[
    \dist\!\left(
    0,
    \partial_C F(\hat\theta(\theta))
    \right)
    \le
    \|\nabla F_{\hat\lambda}(\theta)\|.
    \]
\end{enumerate}
\end{lemma}

\begin{proof}
The Moreau-envelope properties above are standard for proper lower semicontinuous
\(\kappa\)-weakly convex functions with parameter \(\hat\lambda<1/\kappa\); see, e.g.,
\citet{davis2019stochastic} and \citet{drusvyatskiy2018error}. 

By \Cref{prop:online-wc}, \(J_{\rob}^{\mathrm{on}}\) is \(\kappa\)-weakly convex on \(\Theta\).
Since \(\Theta\) is closed and convex, the indicator \(I_\Theta\) is proper, lower semicontinuous,
and convex. Therefore
\[
F=J_{\rob}^{\mathrm{on}}+I_\Theta
\]
is proper, lower semicontinuous, bounded below by assumption, and \(\kappa\)-weakly convex. Here is the detailed statement:

We first prove single-valuedness of the proximal map. Since \(F\) is \(\kappa\)-weakly convex,
the function
\[
u\mapsto F(u)+\frac{\kappa}{2}\|u\|^2
\]
is convex. For fixed \(\theta\), the proximal objective is
\[
u
\mapsto
F(u)+\frac{1}{2\hat\lambda}\|u-\theta\|^2.
\]
Adding and subtracting \(\frac{\kappa}{2}\|u\|^2\), we can write it as
\[
\left(
F(u)+\frac{\kappa}{2}\|u\|^2
\right)
+
\frac{1}{2}
\left(
\frac{1}{\hat\lambda}-\kappa
\right)
\|u\|^2
-
\frac{1}{\hat\lambda}\langle u,\theta\rangle
+
\frac{1}{2\hat\lambda}\|\theta\|^2.
\]
The first term is convex, and the second quadratic term is strongly convex because
\[
\frac{1}{\hat\lambda}-\kappa>0.
\]
Hence the proximal objective is strongly convex. Since it is also proper, lower semicontinuous,
and coercive, it has a unique minimizer. Thus \(\prox_{\hat\lambda F}\) is single-valued.

The standard Moreau-envelope calculus for weakly convex functions then gives
\[
F_{\hat\lambda}\in C^1,
\qquad
\nabla F_{\hat\lambda}(\theta)
=
\hat\lambda^{-1}
\left(
\theta-\prox_{\hat\lambda F}(\theta)
\right)
=
\hat\lambda^{-1}
\left(
\theta-\hat\theta(\theta)
\right).
\]
This proves the differentiability and gradient formula.

The same weakly-convex Moreau calculus gives the Lipschitz bound
\[
\|\nabla F_{\hat\lambda}(\theta)-\nabla F_{\hat\lambda}(\theta')\|
\le
\frac{1}{\hat\lambda(1-\kappa\hat\lambda)}
\|\theta-\theta'\|,
\]
for all \(\theta,\theta'\in\mathbb R^d\). Hence \(\nabla F_{\hat\lambda}\) is Lipschitz with constant
at most
\[
L_{\mathrm{env}}
=
\frac{1}{\hat\lambda(1-\kappa\hat\lambda)}.
\]

It remains to prove the near-stationarity claim. Let
\[
\hat\theta
:=
\hat\theta(\theta)
=
\prox_{\hat\lambda F}(\theta).
\]
By the first-order optimality condition for the proximal problem,
\[
0
\in
\partial_C F(\hat\theta)
+
\frac{1}{\hat\lambda}(\hat\theta-\theta).
\]
Equivalently,
\[
\frac{1}{\hat\lambda}(\theta-\hat\theta)
\in
\partial_C F(\hat\theta).
\]
Using the gradient formula,
\[
\nabla F_{\hat\lambda}(\theta)
=
\frac{1}{\hat\lambda}(\theta-\hat\theta),
\]
we obtain
\[
\nabla F_{\hat\lambda}(\theta)
\in
\partial_C F(\hat\theta(\theta)).
\]
Therefore,
\[
\dist\!\left(
0,
\partial_C F(\hat\theta(\theta))
\right)
\le
\|\nabla F_{\hat\lambda}(\theta)\|.
\]
This proves the lemma.
\end{proof}

\subsection{Subdifferential calculus and oracle (\Cref{lem:oracle})}

\begin{lemma}[Set-valued Danskin]
\label{lem:danskin}
For every \((\theta,x,Y)\), define
\[
M(\theta;x,Y)
:=
\max_{\sigma\in\Sk}
\ell_{\PL}(\theta;x,Y,\sigma).
\]
Then \(M(\cdot;x,Y)\) is convex and Clarke regular. Moreover,
\[
\partial_C M(\theta;x,Y)
=
\conv
\left\{
\nabla_\theta\ell_{\PL}(\theta;x,Y,\sigma)
:
\sigma\in
\argmax_{\sigma'\in\Sk}
\ell_{\PL}(\theta;x,Y,\sigma')
\right\}.
\]
In particular,
\[
\nabla_\theta\ell_{\PL}(\theta;x,Y,\sigma_{\mathrm{worst}})
\in
\partial_C M(\theta;x,Y)
\]
for any \(\sigma_{\mathrm{worst}}\) selected by \Cref{thm:worst-perm}.
\end{lemma}

\begin{proof}
Fix \((x,Y)\) throughout the proof and write
\[
f_\sigma(\theta)
:=
\ell_{\PL}(\theta;x,Y,\sigma),
\qquad
\sigma\in\Sk.
\]
By \Cref{lem:PL-bounds}, each \(f_\sigma\) is convex and continuously differentiable in
\(\theta\). Since \(\Sk\) is finite, the pointwise maximum
\[
M(\theta;x,Y)=\max_{\sigma\in\Sk} f_\sigma(\theta)
\]
is also convex. Moreover, a finite maximum of continuously differentiable functions is locally
Lipschitz. Since every finite-valued convex function is Clarke regular, \(M(\cdot;x,Y)\) is Clarke
regular, and its Clarke subdifferential coincides with its convex subdifferential:
\[
\partial_C M(\theta;x,Y)=\partial M(\theta;x,Y).
\]

It remains to identify this subdifferential. Define the active maximizer set
\[
\mathcal A(\theta;x,Y)
:=
\argmax_{\sigma\in\Sk} f_\sigma(\theta).
\]
This set is nonempty because \(\Sk\) is finite. We prove both inclusions.

First, let \(\sigma\in\mathcal A(\theta;x,Y)\). Since \(f_\sigma\) is convex and differentiable,
for every \(u\),
\[
f_\sigma(u)
\ge
f_\sigma(\theta)
+
\left\langle
\nabla_\theta f_\sigma(\theta),u-\theta
\right\rangle.
\]
Because \(\sigma\) is active, \(f_\sigma(\theta)=M(\theta;x,Y)\). Also,
\(M(u;x,Y)\ge f_\sigma(u)\). Hence
\[
M(u;x,Y)
\ge
M(\theta;x,Y)
+
\left\langle
\nabla_\theta f_\sigma(\theta),u-\theta
\right\rangle
\qquad
\text{for all }u.
\]
Therefore,
\[
\nabla_\theta f_\sigma(\theta)\in\partial M(\theta;x,Y).
\]
Since \(\partial M(\theta;x,Y)\) is convex, we obtain
\[
\conv
\left\{
\nabla_\theta f_\sigma(\theta):
\sigma\in\mathcal A(\theta;x,Y)
\right\}
\subseteq
\partial M(\theta;x,Y).
\]

We now prove the reverse inclusion. For any direction \(d\), the one-sided directional derivative of
\(M\) at \(\theta\) satisfies
\[
M'(\theta;d)
=
\lim_{t\downarrow 0}
\frac{M(\theta+td;x,Y)-M(\theta;x,Y)}{t}.
\]
We claim that
\[
M'(\theta;d)
=
\max_{\sigma\in\mathcal A(\theta;x,Y)}
\left\langle
\nabla_\theta f_\sigma(\theta),d
\right\rangle.
\]
To see this, note first that for every active \(\sigma\in\mathcal A(\theta;x,Y)\),
\[
M(\theta+td;x,Y)
\ge
f_\sigma(\theta+td),
\]
and therefore
\[
\liminf_{t\downarrow 0}
\frac{M(\theta+td;x,Y)-M(\theta;x,Y)}{t}
\ge
\left\langle
\nabla_\theta f_\sigma(\theta),d
\right\rangle.
\]
Taking the maximum over active \(\sigma\) gives the lower bound.

For the upper bound, choose for each \(t>0\) a maximizer
\[
\sigma_t\in
\argmax_{\sigma\in\Sk}
f_\sigma(\theta+td).
\]
Since \(\Sk\) is finite, along any sequence \(t_n\downarrow 0\) there is a subsequence, still denoted
\(t_n\), such that \(\sigma_{t_n}=\bar\sigma\) is constant. By continuity,
\[
f_{\bar\sigma}(\theta)
=
\lim_{n\to\infty} f_{\sigma_{t_n}}(\theta+t_n d)
=
\lim_{n\to\infty} M(\theta+t_n d;x,Y)
=
M(\theta;x,Y),
\]
so \(\bar\sigma\in\mathcal A(\theta;x,Y)\). Hence, along this subsequence,
\[
\lim_{n\to\infty}
\frac{M(\theta+t_n d;x,Y)-M(\theta;x,Y)}{t_n}
=
\left\langle
\nabla_\theta f_{\bar\sigma}(\theta),d
\right\rangle
\le
\max_{\sigma\in\mathcal A(\theta;x,Y)}
\left\langle
\nabla_\theta f_\sigma(\theta),d
\right\rangle.
\]
Since this argument applies to every vanishing sequence \(t_n\downarrow 0\), the claimed directional
derivative formula follows.

Now let \(v\in\partial M(\theta;x,Y)\). By the characterization of the convex subdifferential through
directional derivatives,
\[
\langle v,d\rangle
\le
M'(\theta;d)
\qquad
\text{for every direction }d.
\]
Using the formula above,
\[
\langle v,d\rangle
\le
\max_{\sigma\in\mathcal A(\theta;x,Y)}
\left\langle
\nabla_\theta f_\sigma(\theta),d
\right\rangle
\qquad
\text{for every }d.
\]
We show that this implies
\[
v\in
\conv
\left\{
\nabla_\theta f_\sigma(\theta):
\sigma\in\mathcal A(\theta;x,Y)
\right\}.
\]
Indeed, if \(v\) were not in this closed convex hull, then by the finite-dimensional separating
hyperplane theorem there would exist a direction \(d\) such that
\[
\langle v,d\rangle
>
\max_{\sigma\in\mathcal A(\theta;x,Y)}
\left\langle
\nabla_\theta f_\sigma(\theta),d
\right\rangle,
\]
contradicting the previous inequality. Therefore,
\[
\partial M(\theta;x,Y)
\subseteq
\conv
\left\{
\nabla_\theta f_\sigma(\theta):
\sigma\in\mathcal A(\theta;x,Y)
\right\}.
\]
Combining the two inclusions gives
\[
\partial_C M(\theta;x,Y)
=
\partial M(\theta;x,Y)
=
\conv
\left\{
\nabla_\theta\ell_{\PL}(\theta;x,Y,\sigma)
:
\sigma\in
\argmax_{\sigma'\in\Sk}
\ell_{\PL}(\theta;x,Y,\sigma')
\right\}.
\]

Finally, by \Cref{thm:worst-perm}, any selected worst-case ranking
\(\sigma_{\mathrm{worst}}\) belongs to the active maximizer set. Therefore its gradient is one of the
active gradients and hence belongs to \(\partial_C M(\theta;x,Y)\).
\end{proof}

\begin{lemma}[Clarke regularity and product rule]
\label{lem:clarke-product}
Suppose \cref{ass:offline,ass:online} hold, the following statements hold.

\textup{(a)} For every \((x,Y,\sigma^\star)\),
\(\ell_{\rob}(\cdot;x,Y,\sigma^\star)\) is convex and Clarke regular. Moreover,
\[
\partial_C\ell_{\rob}(\theta;x,Y,\sigma^\star)
=
(1-\rho)\nabla_\theta\ell_{\PL}(\theta;x,Y,\sigma^\star)
+
\rho\,\partial_C M(\theta;x,Y),
\]
where
\[
M(\theta;x,Y)
:=
\max_{\sigma\in\Sk}
\ell_{\PL}(\theta;x,Y,\sigma).
\]

\textup{(b)} For every \((x,Y)\),
\[
L_{\rob}(\theta;x,Y)
:=
\E_{\sigma^\star\sim p^\star(\cdot\mid x,Y)}
[
\ell_{\rob}(\theta;x,Y,\sigma^\star)
]
\]
is convex and Clarke regular. Moreover,
\[
\partial_C L_{\rob}(\theta;x,Y)
=
\E_{\sigma^\star\sim p^\star(\cdot\mid x,Y)}
[
\partial_C\ell_{\rob}(\theta;x,Y,\sigma^\star)
].
\]

\textup{(c)} For every \((x,Y)\), the map
\[
\theta
\mapsto
P_{\pi_\theta}(Y\mid x)L_{\rob}(\theta;x,Y)
\]
is Clarke regular, and
\[
\partial_C
\left[
P_{\pi_\theta}(Y\mid x)L_{\rob}(\theta;x,Y)
\right]
=
P_{\pi_\theta}(Y\mid x)\partial_C L_{\rob}(\theta;x,Y)
+
L_{\rob}(\theta;x,Y)\nabla_\theta P_{\pi_\theta}(Y\mid x).
\]
\end{lemma}

\begin{proof}
We prove the three claims separately.

\textup{(a)}
Fix \((x,Y,\sigma^\star)\). By the robust-TV decomposition,
\[
\ell_{\rob}(\theta;x,Y,\sigma^\star)
=
(1-\rho)\ell_{\PL}(\theta;x,Y,\sigma^\star)
+
\rho M(\theta;x,Y),
\]
where
\[
M(\theta;x,Y)
=
\max_{\sigma\in\Sk}
\ell_{\PL}(\theta;x,Y,\sigma).
\]
By \Cref{lem:PL-bounds}, for every \(\sigma\in\Sk\), the map
\[
\theta\mapsto \ell_{\PL}(\theta;x,Y,\sigma)
\]
is convex and continuously differentiable. Therefore
\(\ell_{\PL}(\cdot;x,Y,\sigma^\star)\) is convex and Clarke regular. By \Cref{lem:danskin},
\(M(\cdot;x,Y)\) is also convex and Clarke regular. Since \(\rho\in[0,1]\), the function
\(\ell_{\rob}(\cdot;x,Y,\sigma^\star)\) is a nonnegative linear combination of convex functions,
and hence is convex. Since it is finite-valued and convex on the parameter space, it is Clarke
regular.

It remains to compute the subdifferential. The convex subdifferential sum rule gives
\[
\partial \ell_{\rob}(\theta;x,Y,\sigma^\star)
=
(1-\rho)\partial \ell_{\PL}(\theta;x,Y,\sigma^\star)
+
\rho\,\partial M(\theta;x,Y).
\]
Because \(\ell_{\PL}(\cdot;x,Y,\sigma^\star)\) is differentiable,
\[
\partial \ell_{\PL}(\theta;x,Y,\sigma^\star)
=
\{\nabla_\theta\ell_{\PL}(\theta;x,Y,\sigma^\star)\}.
\]
Moreover, since all the functions involved are convex and Clarke regular, their Clarke
subdifferentials coincide with their convex subdifferentials. Hence
\[
\partial_C\ell_{\rob}(\theta;x,Y,\sigma^\star)
=
(1-\rho)\nabla_\theta\ell_{\PL}(\theta;x,Y,\sigma^\star)
+
\rho\,\partial_C M(\theta;x,Y).
\]
This proves \textup{(a)}.

\textup{(b)}
Fix \((x,Y)\). Since \(\Sk\) is finite, the conditional expectation over
\(\sigma^\star\sim p^\star(\cdot\mid x,Y)\) is a finite weighted sum:
\[
L_{\rob}(\theta;x,Y)
=
\sum_{\sigma^\star\in\Sk}
p^\star(\sigma^\star\mid x,Y)
\ell_{\rob}(\theta;x,Y,\sigma^\star).
\]
By part \textup{(a)}, each summand is convex and Clarke regular. The weights
\(p^\star(\sigma^\star\mid x,Y)\) are nonnegative, sum to one, and are independent of
\(\theta\). Therefore \(L_{\rob}(\cdot;x,Y)\) is convex and finite-valued, hence Clarke regular.

For the subdifferential, the finite convex-sum rule yields
\[
\partial L_{\rob}(\theta;x,Y)
=
\sum_{\sigma^\star\in\Sk}
p^\star(\sigma^\star\mid x,Y)
\partial \ell_{\rob}(\theta;x,Y,\sigma^\star).
\]
Equivalently,
\[
\partial_C L_{\rob}(\theta;x,Y)
=
\sum_{\sigma^\star\in\Sk}
p^\star(\sigma^\star\mid x,Y)
\partial_C \ell_{\rob}(\theta;x,Y,\sigma^\star).
\]
Writing the finite weighted sum in expectation notation gives
\[
\partial_C L_{\rob}(\theta;x,Y)
=
\E_{\sigma^\star\sim p^\star(\cdot\mid x,Y)}
[
\partial_C\ell_{\rob}(\theta;x,Y,\sigma^\star)
].
\]
This is the finite-dimensional Aumann identity in the present setting. Since the label space
\(\Sk\) is finite and the subgradients are uniformly bounded by \Cref{lem:PL-bounds},
measurability and integrability are automatic. This proves \textup{(b)}.

\textup{(c)}
Fix \((x,Y)\) and define
\[
q(\theta)
:=
P_{\pi_\theta}(Y\mid x),
\qquad
L(\theta)
:=
L_{\rob}(\theta;x,Y).
\]
Under \cref{ass:offline}, \(q\) is continuously
differentiable in \(\theta\). In fact, because the softmax policy assigns strictly positive
probability to every response in the finite response set,
\[
q(\theta)>0.
\]
By part \textup{(b)}, \(L\) is convex and Clarke regular. Moreover, \(L(\theta)\ge 0\), since
\(\ell_{\PL}\ge 0\) and \(\ell_{\rob}\) is a convex combination of nonnegative PL losses.

We first verify Clarke regularity of the product \(qL\). For any direction \(d\), since \(q\) is
\(C^1\) and \(L\) is directionally differentiable and Clarke regular,
\[
(qL)'(\theta;d)
=
\langle \nabla q(\theta),d\rangle L(\theta)
+
q(\theta)L'(\theta;d).
\]
Because \(L\) is Clarke regular,
\[
L'(\theta;d)
=
\max_{v\in\partial_C L(\theta)}
\langle v,d\rangle.
\]
Therefore
\[
(qL)'(\theta;d)
=
\max_{v\in\partial_C L(\theta)}
\left\langle
L(\theta)\nabla q(\theta)+q(\theta)v,
d
\right\rangle.
\]
The right-hand side is the support function of the compact convex set
\[
L(\theta)\nabla q(\theta)
+
q(\theta)\partial_C L(\theta).
\]
Hence the Clarke directional derivative of \(qL\) agrees with its ordinary directional derivative,
so \(qL\) is Clarke regular. Its Clarke subdifferential is exactly the set whose support function
appears above:
\[
\partial_C(qL)(\theta)
=
L(\theta)\nabla q(\theta)
+
q(\theta)\partial_C L(\theta).
\]
Substituting back
\[
q(\theta)=P_{\pi_\theta}(Y\mid x),
\qquad
L(\theta)=L_{\rob}(\theta;x,Y),
\]
we obtain
\[
\partial_C
\left[
P_{\pi_\theta}(Y\mid x)L_{\rob}(\theta;x,Y)
\right]
=
P_{\pi_\theta}(Y\mid x)\partial_C L_{\rob}(\theta;x,Y)
+
L_{\rob}(\theta;x,Y)\nabla_\theta P_{\pi_\theta}(Y\mid x).
\]
This proves \textup{(c)}.
\end{proof}

\begin{lemma}[Score-function identity]
\label{lem:score-id}
Suppose \cref{ass:offline,ass:online} hold, \(J_{\rob}^{\mathrm{on}}\) is locally Lipschitz on \(\Theta\) with constant \(\le C_G+2KB_\psi C_L\), and
\begin{equation}
\partial_C J_{\rob}^{\mathrm{on}}(\theta)
\;\supseteq\;
\E_x\,\E_{Y\sim\pi_\theta^{\otimes K}}
\big[
\partial_C L_{\rob}(\theta;x,Y)
+
L_{\rob}(\theta;x,Y)S_\theta(x,Y)
\big].
\label{eq:score-id-app}
\end{equation}
\end{lemma}

\begin{proof}
Fix \(\theta\in\Theta\). The set \(\mathcal Y\) is finite, and hence
\[
J_{\rob}^{\mathrm{on}}(\theta)
=
\E_{x\sim\mathcal D_x}
\left[
H(\theta;x)
\right],
\]
where
\[
H(\theta;x)
:=
\sum_{Y\in\mathcal Y^K}
P_\theta(Y\mid x)
L_{\rob}(\theta;x,Y),
\qquad
P_\theta(Y\mid x)
:=
\pi_\theta^{\otimes K}(Y\mid x).
\]
Writing \(Y=(y_1,\ldots,y_K)\), the policy-induced list probability is
\[
P_\theta(Y\mid x)
=
\prod_{i=1}^K \pi_\theta(y_i\mid x).
\]
Therefore,
\[
\nabla_\theta \log P_\theta(Y\mid x)
=
\sum_{i=1}^K
\nabla_\theta \log \pi_\theta(y_i\mid x)
=:S_\theta(x,Y).
\]
Since \(P_\theta(Y\mid x)\) is continuously differentiable in \(\theta\), we obtain the score-function identity
\begin{equation}
\label{eq:score-prob-identity}
\nabla_\theta P_\theta(Y\mid x)
=
P_\theta(Y\mid x)S_\theta(x,Y).
\end{equation}

We first establish the subdifferential inclusion. Fix \(x\) and \(Y\). By construction,
\(L_{\rob}(\theta;x,Y)\) is locally Lipschitz and Clarke regular in \(\theta\). Let
\[
V(\theta;x,Y)\in \partial_C L_{\rob}(\theta;x,Y)
\]
be an arbitrary measurable selection. Since \(P_\theta(Y\mid x)\) is \(C^1\), the Clarke product rule in \Cref{lem:clarke-product}(c) gives
\[
P_\theta(Y\mid x)V(\theta;x,Y)
+
L_{\rob}(\theta;x,Y)\nabla_\theta P_\theta(Y\mid x)
\in
\partial_C
\Big(
P_\theta(Y\mid x)L_{\rob}(\theta;x,Y)
\Big).
\]
Using \eqref{eq:score-prob-identity}, this becomes
\[
P_\theta(Y\mid x)
\Big[
V(\theta;x,Y)
+
L_{\rob}(\theta;x,Y)S_\theta(x,Y)
\Big]
\in
\partial_C
\Big(
P_\theta(Y\mid x)L_{\rob}(\theta;x,Y)
\Big).
\]
Since \(\mathcal Y^K\) is finite, we may sum over \(Y\). By the Clarke sum rule,
\[
\sum_{Y\in\mathcal Y^K}
P_\theta(Y\mid x)
\Big[
V(\theta;x,Y)
+
L_{\rob}(\theta;x,Y)S_\theta(x,Y)
\Big]
\in
\partial_C H(\theta;x).
\]
Equivalently,
\[
\E_{Y\sim\pi_\theta^{\otimes K}(\cdot\mid x)}
\Big[
V(\theta;x,Y)
+
L_{\rob}(\theta;x,Y)S_\theta(x,Y)
\Big]
\in
\partial_C H(\theta;x).
\]

Now take expectation over \(x\sim\mathcal D_x\). The selection above is measurable by the deterministic tie-breaking rule used in the definition of the worst-case ranking, and it is integrable by the uniform bounds on \(\partial_C L_{\rob}\), \(L_{\rob}\), and \(S_\theta\). Hence the Aumann expectation rule yields
\[
\E_x
\E_{Y\sim\pi_\theta^{\otimes K}(\cdot\mid x)}
\Big[
V(\theta;x,Y)
+
L_{\rob}(\theta;x,Y)S_\theta(x,Y)
\Big]
\in
\partial_C J_{\rob}^{\mathrm{on}}(\theta).
\]
Since \(V(\theta;x,Y)\) was an arbitrary measurable selection from
\(\partial_C L_{\rob}(\theta;x,Y)\), this proves the set-valued inclusion
\[
\partial_C J_{\rob}^{\mathrm{on}}(\theta)
\;\supseteq\;
\E_x\,\E_{Y\sim\pi_\theta^{\otimes K}}
\big[
\partial_C L_{\rob}(\theta;x,Y)
+
L_{\rob}(\theta;x,Y)S_\theta(x,Y)
\big].
\]

It remains to verify the stated Lipschitz bound. By the PL gradient bound and the robust-loss construction,
\[
\sup_{V\in\partial_C L_{\rob}(\theta;x,Y)}
\|V\|
\le C_G.
\]
Moreover,
\[
0\le L_{\rob}(\theta;x,Y)\le C_L,
\qquad
\|S_\theta(x,Y)\|\le 2KB_\psi.
\]
Therefore, every vector in the set
\[
\partial_C L_{\rob}(\theta;x,Y)
+
L_{\rob}(\theta;x,Y)S_\theta(x,Y)
\]
has norm at most
\[
C_G+2KB_\psi C_L.
\]
Averaging over \(Y\sim\pi_\theta^{\otimes K}(\cdot\mid x)\) and then over \(x\sim\mathcal D_x\) preserves this bound. Hence every element constructed in the right-hand side of
\eqref{eq:score-id-app} has norm at most \(C_G+2KB_\psi C_L\). Since \(J_{\rob}^{\mathrm{on}}\) is locally Lipschitz and its Clarke subgradients are uniformly bounded by this quantity, \(J_{\rob}^{\mathrm{on}}\) is locally Lipschitz on \(\Theta\) with constant at most
\[
C_G+2KB_\psi C_L.
\]
The same argument holds on any bounded neighborhood of \(\Theta\), with \(D\) in \(C_L\) replaced by the corresponding radius on that neighborhood.
\end{proof}

\begin{definition}[Measurable selector and oracle]
\label{def:oracle}
Fix any deterministic linear order $\preceq$ on $\Sk$ and let $\sigma^{\mathrm{sel}}(\theta;x,Y) := \min_\preceq\argmax_\sigma\ell_{\PL}(\theta;x,Y,\sigma)$, a Borel-measurable selection computable by ascending sort with $\preceq$ as tie-breaking. The per-sample oracle is~\eqref{eq:oracle}.
\end{definition}

\begin{lemma}[Oracle properties]
\label{lem:oracle}
Suppose \cref{ass:offline,ass:online} hold: 
\textup{(a)} \(G_1\in\partial_C\ell_{\rob}(\theta;x,Y,\sigma^\star)\) for every \((\theta,x,Y,\sigma^\star)\). 
\textup{(b)} \(\E[G(\theta;Z)\mid\theta]\in\partial_C J_{\rob}^{\mathrm{on}}(\theta)\). 
\textup{(c)} \(\E[\|G\|^2\mid\theta]\le G_{\mathrm{tot}}^2\), where \(G_{\mathrm{tot}}^2\) is defined below.
\begin{equation}
\label{eq:Gtot}
G_{\mathrm{tot}}^2
:=
2C_G^2
+
2(2KB_\psi C_L)^2,
\qquad
C_G:=2KB_\psi,
\qquad
C_L:=K(\log K+2DB_\psi).
\end{equation}
\end{lemma}

\begin{proof}
Fix \(\theta\in\Theta\). Throughout the proof, all expectations are conditional on the current
parameter \(\theta\) unless otherwise stated. Recall that
\[
G(\theta;Z)=G_1(\theta;Z)+G_2(\theta;Z),
\]
where
\[
G_1(\theta;Z)
:=
(1-\rho)\nabla_\theta \ell_{\PL}(\theta;x,Y,\sigma^\star)
+
\rho \nabla_\theta \ell_{\PL}(\theta;x,Y,\sigma^{\mathrm{sel}}),
\]
and
\[
G_2(\theta;Z)
:=
\ell_{\rob}(\theta;x,Y,\sigma^\star)S_\theta(x,Y).
\]

\textup{(a)}
Fix \((x,Y,\sigma^\star)\). Define the pointwise worst-case PL loss
\[
M(\theta;x,Y)
:=
\max_{\sigma\in\Sk}\ell_{\PL}(\theta;x,Y,\sigma).
\]
Then
\[
\ell_{\rob}(\theta;x,Y,\sigma^\star)
=
(1-\rho)\ell_{\PL}(\theta;x,Y,\sigma^\star)
+
\rho M(\theta;x,Y).
\]
Since \(\Sk\) is finite and, for each \(\sigma\in\Sk\),
\(\theta\mapsto \ell_{\PL}(\theta;x,Y,\sigma)\) is continuously differentiable,
the function \(M(\theta;x,Y)\) is locally Lipschitz and Clarke regular. By the
finite-max Danskin theorem in \Cref{lem:danskin},
\[
\partial_C M(\theta;x,Y)
=
\operatorname{conv}
\left\{
\nabla_\theta \ell_{\PL}(\theta;x,Y,\sigma)
:
\sigma\in
\arg\max_{\tilde\sigma\in\Sk}
\ell_{\PL}(\theta;x,Y,\tilde\sigma)
\right\}.
\]
By construction, \(\sigma^{\mathrm{sel}}\) is selected from the active maximizer set:
\[
\sigma^{\mathrm{sel}}
\in
\arg\max_{\sigma\in\Sk}
\ell_{\PL}(\theta;x,Y,\sigma).
\]
Therefore,
\[
\nabla_\theta\ell_{\PL}(\theta;x,Y,\sigma^{\mathrm{sel}})
\in
\partial_C M(\theta;x,Y).
\]
Using the Clarke sum rule in \Cref{lem:clarke-product}(a), we obtain
\[
(1-\rho)\nabla_\theta\ell_{\PL}(\theta;x,Y,\sigma^\star)
+
\rho\nabla_\theta\ell_{\PL}(\theta;x,Y,\sigma^{\mathrm{sel}})
\in
\partial_C\ell_{\rob}(\theta;x,Y,\sigma^\star).
\]
The left-hand side is exactly \(G_1(\theta;Z)\). Hence
\[
G_1(\theta;Z)\in\partial_C\ell_{\rob}(\theta;x,Y,\sigma^\star).
\]
This argument also covers ties: when the active maximizer set is not a singleton,
the deterministic tie-breaking rule selects one active maximizer, and every active
gradient is a valid element of the Clarke subdifferential of the finite maximum.

\textup{(b)}
For fixed \((x,Y)\), define the conditional robust loss
\[
L_{\rob}(\theta;x,Y)
:=
\E_{\sigma^\star\sim p^\star(\cdot|x,Y)}
\left[
\ell_{\rob}(\theta;x,Y,\sigma^\star)
\right].
\]
Since \(\Sk\) is finite, the expectation over \(\sigma^\star\) is a finite sum. Moreover,
by part \textup{(a)}, for every possible \(\sigma^\star\),
\[
G_1(\theta;x,Y,\sigma^\star)
\in
\partial_C\ell_{\rob}(\theta;x,Y,\sigma^\star).
\]
Therefore, by the Aumann expectation rule in \Cref{lem:clarke-product}(b),
\[
H(\theta;x,Y)
:=
\E_{\sigma^\star}
\left[
G_1(\theta;Z)
\,\middle|\,
x,Y,\theta
\right]
\in
\partial_C L_{\rob}(\theta;x,Y).
\]
The measurability of \(H\) follows from the deterministic tie-breaking rule used to define
\(\sigma^{\mathrm{sel}}\), and integrability follows from the uniform gradient bound.

Next, since \(S_\theta(x,Y)\) depends only on the policy-generated list \((x,Y)\) and not on
the oracle label \(\sigma^\star\), we have
\[
\E_{\sigma^\star}
\left[
G_2(\theta;Z)
\,\middle|\,
x,Y,\theta
\right]
=
\E_{\sigma^\star}
\left[
\ell_{\rob}(\theta;x,Y,\sigma^\star)
\,\middle|\,
x,Y,\theta
\right]
S_\theta(x,Y)
=
L_{\rob}(\theta;x,Y)S_\theta(x,Y).
\]
Combining the two conditional expectations gives
\[
\E_{\sigma^\star}
\left[
G(\theta;Z)
\,\middle|\,
x,Y,\theta
\right]
=
H(\theta;x,Y)
+
L_{\rob}(\theta;x,Y)S_\theta(x,Y),
\]
with
\[
H(\theta;x,Y)\in\partial_C L_{\rob}(\theta;x,Y).
\]

Now take expectation over \(x\sim\mathcal D_x\) and
\(Y\sim\pi_\theta^{\otimes K}(\cdot|x)\). We obtain
\[
\E[G(\theta;Z)\mid\theta]
=
\E_{x,Y}
\left[
H(\theta;x,Y)
+
L_{\rob}(\theta;x,Y)S_\theta(x,Y)
\right],
\]
where \(H(\theta;x,Y)\in\partial_C L_{\rob}(\theta;x,Y)\) is a measurable
selection. By the score-function identity in \Cref{lem:score-id},
\[
\E_{x,Y}
\left[
H(\theta;x,Y)
+
L_{\rob}(\theta;x,Y)S_\theta(x,Y)
\right]
\in
\partial_C J_{\rob}^{\mathrm{on}}(\theta).
\]
Therefore,
\[
\E[G(\theta;Z)\mid\theta]
\in
\partial_C J_{\rob}^{\mathrm{on}}(\theta).
\]

\textup{(c)}
By the PL gradient bound, for every ranking \(\sigma\in\Sk\),
\[
\|\nabla_\theta\ell_{\PL}(\theta;x,Y,\sigma)\|
\le C_G.
\]
Since \(G_1\) is a convex combination of two PL gradients, we have
\[
\|G_1\|
\le
(1-\rho)
\|\nabla_\theta\ell_{\PL}(\theta;x,Y,\sigma^\star)\|
+
\rho
\|\nabla_\theta\ell_{\PL}(\theta;x,Y,\sigma^{\mathrm{sel}})\|
\le
(1-\rho)C_G+\rho C_G
=
C_G.
\]
Moreover, by the uniform loss bound and the score-function bound,
\[
0\le
\ell_{\rob}(\theta;x,Y,\sigma^\star)
\le C_L,
\qquad
\|S_\theta(x,Y)\|
\le 2KB_\psi.
\]
Hence
\[
\|G_2\|
=
\|\ell_{\rob}(\theta;x,Y,\sigma^\star)S_\theta(x,Y)\|
\le
2KB_\psi C_L.
\]
Using the elementary inequality
\[
\|a+b\|^2\le 2\|a\|^2+2\|b\|^2,
\]
we obtain the pointwise bound
\[
\|G(\theta;Z)\|^2
\le
2\|G_1\|^2+2\|G_2\|^2
\le
2C_G^2
+
2(2KB_\psi C_L)^2
=
G_{\mathrm{tot}}^2.
\]
Taking conditional expectation given \(\theta\) preserves the bound:
\[
\E[\|G(\theta;Z)\|^2\mid\theta]
\le
G_{\mathrm{tot}}^2.
\]
This proves \textup{(c)}.
\end{proof}
\begin{remark}
No measure-zero / almost-sure caveat is needed: \eqref{eq:oracle} is valid for every $(\theta,x,Y,\sigma^\star)$, even at score ties — which can occur with positive probability under iid sampling with replacement.
\end{remark}

\subsection{One-step descent and convergence (\Cref{thm:online-rate}, \Cref{cor:complexity})}

\begin{lemma}[One-step Moreau-envelope descent]
\label{lem:descent}
Suppose \cref{ass:offline,ass:online} hold, the conditions of \Cref{prop:online-wc}, \Cref{lem:moreau}, and
\Cref{lem:oracle} hold. Let \(\hat\lambda\in(0,1/\kappa)\). Consider the iterates of
\Cref{alg:online-robust-pl-sail},
\[
\widehat G_t
=
\frac{1}{B_s}
\sum_{i=1}^{B_s}
G(\theta_t;Z_i),
\qquad
\theta_{t+1}
=
\Pi_\Theta(\theta_t-\eta\widehat G_t),
\]
where \(Z_i=(x_i,Y_i,\sigma_i^\star)\) are sampled as in the algorithm. Then, for every
\(\eta>0\),
\begin{equation}
\E\big[
F_{\hat\lambda}(\theta_{t+1})
\,\big|\,
\theta_t
\big]
\le
F_{\hat\lambda}(\theta_t)
-
\frac{\eta(1-\kappa\hat\lambda)}{2}
\|\nabla F_{\hat\lambda}(\theta_t)\|^2
+
\frac{\eta^2G_{\mathrm{tot}}^2}{2\hat\lambda}.
\label{eq:descent}
\end{equation}
In fact, the stronger bound with coefficient
\(\eta(1-\kappa\hat\lambda)\) instead of
\(\eta(1-\kappa\hat\lambda)/2\) also holds.
\end{lemma}

\begin{proof}
We first record the mini-batch oracle properties. Conditional on \(\theta_t\), the samples
\(Z_1,\ldots,Z_{B_s}\) are iid from the policy-induced online sampling distribution. By
\Cref{lem:oracle}(b),
\[
\E[G(\theta_t;Z_i)\mid\theta_t]
\in
\partial_CJ_{\rob}^{\mathrm{on}}(\theta_t)
\qquad
\text{for every }i.
\]
Since the Clarke subdifferential is convex,
\[
\bar G_t
:=
\E[\widehat G_t\mid\theta_t]
=
\frac1{B_s}
\sum_{i=1}^{B_s}
\E[G(\theta_t;Z_i)\mid\theta_t]
\in
\partial_CJ_{\rob}^{\mathrm{on}}(\theta_t).
\]
Moreover, by Jensen's inequality and \Cref{lem:oracle}(c),
\[
\E[\|\widehat G_t\|^2\mid\theta_t]
=
\E\left[
\left\|
\frac1{B_s}\sum_{i=1}^{B_s}G(\theta_t;Z_i)
\right\|^2
\,\middle|\,
\theta_t
\right]
\le
\frac1{B_s}\sum_{i=1}^{B_s}
\E[
\|G(\theta_t;Z_i)\|^2
\mid
\theta_t
]
\le
G_{\mathrm{tot}}^2.
\]

Let
\[
\hat\theta_t
:=
\hat\theta(\theta_t)
=
\prox_{\hat\lambda F}(\theta_t).
\]
Since \(\hat\theta_t\in\Theta\), it is feasible for the constrained Moreau envelope at
\(\theta_{t+1}\). Therefore,
\[
F_{\hat\lambda}(\theta_{t+1})
\le
J_{\rob}^{\mathrm{on}}(\hat\theta_t)
+
\frac{1}{2\hat\lambda}
\|\hat\theta_t-\theta_{t+1}\|^2.
\]
By definition of \(\hat\theta_t\),
\[
F_{\hat\lambda}(\theta_t)
=
J_{\rob}^{\mathrm{on}}(\hat\theta_t)
+
\frac{1}{2\hat\lambda}
\|\hat\theta_t-\theta_t\|^2.
\]
Subtracting yields
\begin{equation}
F_{\hat\lambda}(\theta_{t+1})-F_{\hat\lambda}(\theta_t)
\le
\frac{1}{2\hat\lambda}
\left(
\|\hat\theta_t-\theta_{t+1}\|^2
-
\|\hat\theta_t-\theta_t\|^2
\right).
\label{eq:env-step}
\end{equation}

Because \(\Theta\) is closed and convex, the projection \(\Pi_\Theta\) is nonexpansive. Since
\(\hat\theta_t\in\Theta\),
\[
\|\theta_{t+1}-\hat\theta_t\|^2
=
\|\Pi_\Theta(\theta_t-\eta\widehat G_t)-\Pi_\Theta(\hat\theta_t)\|^2
\le
\|\theta_t-\eta\widehat G_t-\hat\theta_t\|^2.
\]
Expanding the right-hand side,
\begin{equation}
\|\theta_{t+1}-\hat\theta_t\|^2
\le
\|\theta_t-\hat\theta_t\|^2
-
2\eta\langle \widehat G_t,\theta_t-\hat\theta_t\rangle
+
\eta^2\|\widehat G_t\|^2.
\label{eq:proj-expand}
\end{equation}
Combining \eqref{eq:env-step} and \eqref{eq:proj-expand}, then taking conditional expectation
given \(\theta_t\), gives
\begin{align}
\E[
F_{\hat\lambda}(\theta_{t+1})-F_{\hat\lambda}(\theta_t)
\mid\theta_t
]
&\le
-\frac{\eta}{\hat\lambda}
\langle \bar G_t,\theta_t-\hat\theta_t\rangle
+
\frac{\eta^2}{2\hat\lambda}
\E[\|\widehat G_t\|^2\mid\theta_t]
\notag\\
&\le
-\frac{\eta}{\hat\lambda}
\langle \bar G_t,\theta_t-\hat\theta_t\rangle
+
\frac{\eta^2G_{\mathrm{tot}}^2}{2\hat\lambda}.
\label{eq:env-cond}
\end{align}

It remains to lower bound
\(\langle \bar G_t,\theta_t-\hat\theta_t\rangle\).
The proximal point satisfies
\[
\hat\theta_t
\in
\argmin_{u\in\Theta}
\left\{
J_{\rob}^{\mathrm{on}}(u)
+
\frac{1}{2\hat\lambda}\|u-\theta_t\|^2
\right\}.
\]
The first-order optimality condition over \(\Theta\) gives some
\[
\zeta_t\in\partial_CJ_{\rob}^{\mathrm{on}}(\hat\theta_t),
\qquad
n_t\in N_\Theta(\hat\theta_t),
\]
such that
\[
0
=
\zeta_t
+
\frac{1}{\hat\lambda}(\hat\theta_t-\theta_t)
+
n_t.
\]
Using the normal-cone convention
\[
\langle n_t,u-\hat\theta_t\rangle\le0,
\qquad
\forall u\in\Theta,
\]
and taking \(u=\theta_t\in\Theta\), we get
\[
\left\langle
\zeta_t
+
\frac{1}{\hat\lambda}(\hat\theta_t-\theta_t),
\theta_t-\hat\theta_t
\right\rangle
=
-\langle n_t,\theta_t-\hat\theta_t\rangle
\ge0.
\]
Therefore,
\begin{equation}
\langle \zeta_t,\theta_t-\hat\theta_t\rangle
\ge
\frac{1}{\hat\lambda}
\|\theta_t-\hat\theta_t\|^2.
\label{eq:prox-foc}
\end{equation}

By \(\kappa\)-weak convexity of \(J_{\rob}^{\mathrm{on}}\), the function
\[
\Phi(\theta)
:=
J_{\rob}^{\mathrm{on}}(\theta)
+
\frac{\kappa}{2}\|\theta\|^2
\]
is convex on \(\Theta\). Since
\[
\bar G_t\in\partial_CJ_{\rob}^{\mathrm{on}}(\theta_t),
\qquad
\zeta_t\in\partial_CJ_{\rob}^{\mathrm{on}}(\hat\theta_t),
\]
we have
\[
\bar G_t+\kappa\theta_t\in\partial_C\Phi(\theta_t),
\qquad
\zeta_t+\kappa\hat\theta_t\in\partial_C\Phi(\hat\theta_t).
\]
By monotonicity of the convex subdifferential of \(\Phi\),
\[
\left\langle
(\bar G_t+\kappa\theta_t)
-
(\zeta_t+\kappa\hat\theta_t),
\theta_t-\hat\theta_t
\right\rangle
\ge0.
\]
Equivalently,
\begin{equation}
\langle \bar G_t-\zeta_t,\theta_t-\hat\theta_t\rangle
\ge
-\kappa\|\theta_t-\hat\theta_t\|^2.
\label{eq:weak-mon}
\end{equation}
Adding \eqref{eq:prox-foc} and \eqref{eq:weak-mon},
\begin{equation}
\langle \bar G_t,\theta_t-\hat\theta_t\rangle
\ge
\left(
\frac1{\hat\lambda}-\kappa
\right)
\|\theta_t-\hat\theta_t\|^2
=
\frac{1-\kappa\hat\lambda}{\hat\lambda}
\|\theta_t-\hat\theta_t\|^2.
\label{eq:gbar-bound}
\end{equation}

By \Cref{lem:moreau},
\[
\nabla F_{\hat\lambda}(\theta_t)
=
\hat\lambda^{-1}
(\theta_t-\hat\theta_t).
\]
Hence
\[
\|\theta_t-\hat\theta_t\|^2
=
\hat\lambda^2
\|\nabla F_{\hat\lambda}(\theta_t)\|^2.
\]
Substituting into \eqref{eq:gbar-bound} gives
\[
\langle \bar G_t,\theta_t-\hat\theta_t\rangle
\ge
(1-\kappa\hat\lambda)\hat\lambda
\|\nabla F_{\hat\lambda}(\theta_t)\|^2.
\]
Plugging this into \eqref{eq:env-cond}, we obtain
\[
\E[
F_{\hat\lambda}(\theta_{t+1})-F_{\hat\lambda}(\theta_t)
\mid\theta_t
]
\le
-\eta(1-\kappa\hat\lambda)
\|\nabla F_{\hat\lambda}(\theta_t)\|^2
+
\frac{\eta^2G_{\mathrm{tot}}^2}{2\hat\lambda}.
\]
Since \(1-\kappa\hat\lambda>0\), this stronger bound implies \eqref{eq:descent}.
\end{proof}

\begin{proof}[Proof of \Cref{thm:online-rate}]
If \(\Delta_0=0\), then \(F_{\hat\lambda}(\theta_0)=F_{\inf}\), and the desired bound is
trivial. Hence assume \(\Delta_0>0\). By \Cref{lem:descent}, for every \(t=0,\ldots,T-1\),
\[
\E\!\left[
F_{\hat\lambda}(\theta_{t+1})
\,\middle|\,
\theta_t
\right]
\le
F_{\hat\lambda}(\theta_t)
-
\frac{\eta(1-\kappa\hat\lambda)}{2}
\|\nabla F_{\hat\lambda}(\theta_t)\|^2
+
\frac{\eta^2G_{\mathrm{tot}}^2}{2\hat\lambda}.
\]
Taking total expectation gives
\[
\E[
F_{\hat\lambda}(\theta_{t+1})
]
\le
\E[
F_{\hat\lambda}(\theta_t)
]
-
\frac{\eta(1-\kappa\hat\lambda)}{2}
\E\|\nabla F_{\hat\lambda}(\theta_t)\|^2
+
\frac{\eta^2G_{\mathrm{tot}}^2}{2\hat\lambda}.
\]
Rearranging,
\[
\frac{\eta(1-\kappa\hat\lambda)}{2}
\E\|\nabla F_{\hat\lambda}(\theta_t)\|^2
\le
\E[
F_{\hat\lambda}(\theta_t)
]
-
\E[
F_{\hat\lambda}(\theta_{t+1})
]
+
\frac{\eta^2G_{\mathrm{tot}}^2}{2\hat\lambda}.
\]
Summing over \(t=0,\ldots,T-1\), we obtain
\[
\frac{\eta(1-\kappa\hat\lambda)}{2}
\sum_{t=0}^{T-1}
\E\|\nabla F_{\hat\lambda}(\theta_t)\|^2
\le
F_{\hat\lambda}(\theta_0)
-
\E[
F_{\hat\lambda}(\theta_T)
]
+
\frac{T\eta^2G_{\mathrm{tot}}^2}{2\hat\lambda}.
\]
Since
\[
F_{\hat\lambda}(\theta)
=
\min_u
\left\{
F(u)+\frac{1}{2\hat\lambda}\|u-\theta\|^2
\right\}
\ge
\inf_u F(u)
=
F_{\inf},
\]
we have
\[
F_{\hat\lambda}(\theta_0)
-
\E[
F_{\hat\lambda}(\theta_T)
]
\le
F_{\hat\lambda}(\theta_0)-F_{\inf}
=
\Delta_0.
\]
Therefore,
\[
\frac{\eta(1-\kappa\hat\lambda)}{2}
\sum_{t=0}^{T-1}
\E\|\nabla F_{\hat\lambda}(\theta_t)\|^2
\le
\Delta_0
+
\frac{T\eta^2G_{\mathrm{tot}}^2}{2\hat\lambda}.
\]
Dividing by \(\eta T\) gives
\[
\frac{1-\kappa\hat\lambda}{2}
\cdot
\frac{1}{T}
\sum_{t=0}^{T-1}
\E\|\nabla F_{\hat\lambda}(\theta_t)\|^2
\le
\frac{\Delta_0}{\eta T}
+
\frac{\eta G_{\mathrm{tot}}^2}{2\hat\lambda}.
\]
With the choice
\[
\eta
=
\sqrt{
\frac{2\hat\lambda\Delta_0}{G_{\mathrm{tot}}^2T}
},
\]
the two terms on the right-hand side are equal:
\[
\frac{\Delta_0}{\eta T}
=
\frac{\eta G_{\mathrm{tot}}^2}{2\hat\lambda}
=
\sqrt{
\frac{\Delta_0G_{\mathrm{tot}}^2}{2\hat\lambda T}
}.
\]
Hence
\[
\frac{1-\kappa\hat\lambda}{2}
\cdot
\frac{1}{T}
\sum_{t=0}^{T-1}
\E\|\nabla F_{\hat\lambda}(\theta_t)\|^2
\le
\sqrt{
\frac{2\Delta_0G_{\mathrm{tot}}^2}{\hat\lambda T}
}.
\]
Let \(R\sim\mathrm{Uniform}\{0,\ldots,T-1\}\) be sampled independently of the algorithmic
randomness. Then
\[
\E\|\nabla F_{\hat\lambda}(\theta_R)\|^2
=
\frac{1}{T}
\sum_{t=0}^{T-1}
\E\|\nabla F_{\hat\lambda}(\theta_t)\|^2.
\]
Therefore,
\[
\E\|\nabla F_{\hat\lambda}(\theta_R)\|^2
\le
\frac{2}{1-\kappa\hat\lambda}
\sqrt{
\frac{2\Delta_0G_{\mathrm{tot}}^2}{\hat\lambda T}
},
\]
which is exactly \eqref{eq:online-rate}.
\end{proof} 

\begin{proof}[Proof of \Cref{cor:complexity}]
By \Cref{thm:online-rate}, it suffices to require
\[
\frac{2}{1-\kappa\hat\lambda}
\sqrt{
\frac{2\Delta_0G_{\mathrm{tot}}^2}{\hat\lambda T}
}
\le
\varepsilon.
\]
Squaring both sides gives
\[
\frac{4}{(1-\kappa\hat\lambda)^2}
\cdot
\frac{2\Delta_0G_{\mathrm{tot}}^2}{\hat\lambda T}
\le
\varepsilon^2.
\]
Equivalently,
\[
T
\ge
\frac{
8\Delta_0G_{\mathrm{tot}}^2
}{
\hat\lambda(1-\kappa\hat\lambda)^2\varepsilon^2
}.
\]
This proves the first claim.

Now set
\[
\hat\lambda=\frac{1}{2\kappa}.
\]
Then
\[
1-\kappa\hat\lambda
=
\frac12,
\]
and hence
\[
T
\ge
\frac{
8\Delta_0G_{\mathrm{tot}}^2
}{
(1/(2\kappa))(1/2)^2\varepsilon^2
}
=
\frac{
64\kappa\Delta_0G_{\mathrm{tot}}^2
}{
\varepsilon^2
}.
\]
Therefore,
\[
T
=
O\left(
\frac{\kappa\Delta_0G_{\mathrm{tot}}^2}{\varepsilon^2}
\right).
\]

It remains to substitute the explicit constants. From \eqref{eq:kappa-listwise},
\[
\kappa
=
K^2B_\psi^2
(8+\log K+2DB_\psi)
=
\tilde O\!\left(
K^2B_\psi^2(\log K+DB_\psi)
\right).
\]
From \eqref{eq:Gtot},
\[
G_{\mathrm{tot}}^2
=
2C_G^2
+
2(2KB_\psi C_L)^2,
\]
where
\[
C_G=2KB_\psi,
\qquad
C_L=K(\log K+2DB_\psi).
\]
Thus
\[
2C_G^2
=
O(K^2B_\psi^2),
\]
and
\[
2(2KB_\psi C_L)^2
=
O\!\left(
K^2B_\psi^2
\cdot
K^2(\log K+DB_\psi)^2
\right)
=
O\!\left(
K^4B_\psi^2(\log K+DB_\psi)^2
\right).
\]
The second term dominates for the relevant asymptotic regime, so
\[
G_{\mathrm{tot}}^2
=
O\!\left(
K^4B_\psi^2(\log K+DB_\psi)^2
\right).
\]
Combining the bounds gives
\[
\kappa G_{\mathrm{tot}}^2
=
\tilde O\!\left(
K^6B_\psi^4(\log K+DB_\psi)^3
\right).
\]
Therefore,
\[
T
=
\tilde O\!\left(
\frac{
K^6B_\psi^4(\log K+DB_\psi)^3
}{
\varepsilon^2
}
\right),
\]
which proves \eqref{eq:online-complexity}.
\end{proof}

\section{\texorpdfstring{SAIL bilevel reduction (used in \Cref{sec:online-theory})}{SAIL bilevel reduction}}
\label{app:sail}

For online alignment with KL-regularized RLHF, the lower-level policy induced by a reward $r$ is $\pi_r^\star(y\mid x) = \pi_{\mathrm{ref}}(y\mid x)\exp(r(x,y)/\beta)/Z_r(x)$. Substituting into the PL likelihood \eqref{eq:pl-prob} and noting that $\log Z_r(x)$ is a per-prompt constant that cancels at every PL stage, the bilevel objective $\min_r\,\E_{x,Y\sim(\pi_r^\star)^{\otimes K},\sigma\sim p^\star}[\ell_{\PL}(r;x,Y,\sigma)]$ subject to the lower-level optimality is exactly equivalent to the single-level objective $\min_\pi\,\E_{x,Y\sim\pi^{\otimes K},\sigma\sim p^\star}[\ell_{\PL}(\pi;x,Y,\sigma)]$, which is similar to SAIL~\citep{ding2024sail}. Parameterizing $\pi=\pi_\theta$ with $s_\theta(x,y)=\log[\pi_\theta(y\mid x)/\pi_{\mathrm{ref}}(y\mid x)]$ recovers \eqref{eq:online-objectives}. At $K=2$ this reduces exactly to the pairwise SAIL/DPO objective.

\paragraph{Gradient.}
Defining $L(\theta;x,Y) := \E_{\sigma\sim p^\star(\cdot\mid x,Y)}[\ell_{\PL}(\theta;x,Y,\sigma)]$, the score-function gradient identity gives
\begin{equation}
\nabla_\theta J(\theta) = \E_{x}\E_{Y\sim\pi_\theta^{\otimes K}}\Big[\nabla_\theta L(\theta;x,Y) + L(\theta;x,Y)\sum_{i=1}^K\nabla_\theta\log\pi_\theta(y_i\mid x)\Big],
\end{equation}
which combined with \Cref{lem:tv-decomp} produces the per-sample stochastic oracle \eqref{eq:oracle}.

\section{Stagewise PL Hessian decomposition}
\label{app:pl-curvature}

\begin{proposition}[Stagewise PL Hessian as a sum of conditional covariances]
\label{prop:pl_curvature_decomposition}
Fix $(x,y_{1:K},\sigma^\star)$ and let $s_\theta(x,y)=\theta^\top\phi(x,y)$. For each stage $i=1,\dots,K-1$ define the remaining set $R_i:=\{\sigma^\star_i,\dots,\sigma^\star_K\}$, the stagewise softmax $p_i^\theta(j) = e^{s_\theta(x,y_j)}/\sum_{m\in R_i} e^{s_\theta(x,y_m)}$, and the stagewise feature mean $\mu_i(\theta) = \sum_{j\in R_i}p_i^\theta(j)\phi(x,y_j)$. Then
\[
\nabla_\theta^2\ell_{\PL}(\theta;x,y_{1:K},\sigma^\star) = \sum_{i=1}^{K-1}\mathrm{Cov}_{j\sim p_i^\theta}[\phi(x,y_j)] \;\succeq\; 0.
\]
\end{proposition}

\begin{proof}
Write $\ell_{\PL}=\sum_{i=1}^{K-1}\ell_i$ where $\ell_i(\theta) = -\theta^\top\phi(x,y_{\sigma^\star_i})+\log\sum_{j\in R_i}\exp(\theta^\top\phi(x,y_j))$ is the negative log-likelihood of the $i$-th stagewise multinomial choice. Differentiating gives $\nabla\ell_i = -\phi(x,y_{\sigma^\star_i})+\mu_i$ and $\nabla^2\ell_i=\sum_{j\in R_i}p_i^\theta(j)(\phi(x,y_j)-\mu_i)(\phi(x,y_j)-\mu_i)^\top = \mathrm{Cov}_{j\sim p_i^\theta}[\phi(x,y_j)]\succeq 0$. Sum over $i$.
\end{proof}

\begin{remark}[Interpretation]
A pairwise BT comparison contributes one binary-choice covariance to the Hessian; a single PL listwise observation contributes one covariance per ranking stage. This is a statement about the local geometry of the underlying observation model, not a Fisher-information dominance result for the final robust objective.
\end{remark}


\section{Additional Experimental Results}
\label{app:additional-exp}

This appendix contains additional experimental results omitted from the main text for space. 
Appendix~\ref{app:rewardbench} reports clean-label external evaluation on RewardBench. 
Appendix~\ref{app:pairwise_dro_sweeps} gives the hyperparameter sweeps used to select the pairwise robust-DPO baselines in Table~\ref{tab:main_kendall_results}. 
Appendix~\ref{app:rho_sensitivity} reports robustness-radius sensitivity for Robust PL, including both noisy-label and clean-label settings. 
Appendix~\ref{app:qwen25-offline} provide full offline metrics on
Qwen2.5.

\subsection{Clean-label external evaluation}
\label{app:rewardbench}

\Cref{tab:rewardbench} evaluates clean-label offline models on RewardBench. The purpose is not to
claim general benchmark dominance, but to check whether the robust correction damages external
alignment quality. The results support the main-text claim that moderate robustification preserves
clean-label model quality.

\begin{table}[h]
\centering
\caption{External RewardBench evaluation for clean-label offline models. We report category-level
accuracy and the average score over Chat, Chat-Hard, Safety, and Reasoning.}
\label{tab:rewardbench}
\resizebox{\textwidth}{!}{
\begin{tabular}{lcclccccc}
\toprule
Model & $K$ & Training & Method & Chat & Chat-Hard & Safety & Reasoning & RB-Avg \\
\midrule
Qwen3-0.6B & 4 & clean & Nominal BT & \textbf{89.1} & 41.0 & 45.7 & 68.6 & 61.1 \\
Qwen3-0.6B & 4 & clean & Nominal PL & 86.3 & 40.6 & 46.2 & 71.5 & 61.2 \\
Qwen3-0.6B & 4 & clean & Robust PL $(\rho=0.05)$ & 88.3 & \textbf{42.5} & \textbf{47.0} & \textbf{72.7} & \textbf{62.6} \\
\midrule
Qwen3-8B & 4 & clean & Nominal BT & 90.5 & \textbf{46.5} & 55.3 & 51.5 & 60.9 \\
Qwen3-8B & 4 & clean & Nominal PL & \textbf{93.3} & 45.0 & 56.9 & 52.9 & \textbf{62.0} \\
Qwen3-8B & 4 & clean & Robust PL $(\rho=0.05)$ & 93.0 & 43.0 & \textbf{58.2} & \textbf{53.6} & 61.96 \\
\bottomrule
\end{tabular}
}
\end{table}

\subsection{Hyperparameter Selection for Pairwise Robust-DPO Baselines}
\label{app:pairwise_dro_sweeps}

For the pairwise robust-DPO baselines in Table~\ref{tab:main_kendall_results}, we select hyperparameters on a moderate noisy development condition: Qwen3-0.6B under near-tie noise with $\epsilon=0.4$. This condition is noisy enough to test robustness, but less extreme than top-rank $\epsilon=1.0$, where loss-level reweighting can directly amplify systematically corrupted pairs. Both TV-DR-DPO and KLDPO are implemented in our pipeline following their loss-level robust DPO objectives. The selected values are $\rho=0.10$ for TV-DR-DPO and $\tau=1.00$ for KLDPO.

\paragraph{Noisy development sweeps.}
Tables~\ref{tab:tvdpo_sweep_neartie} and~\ref{tab:kldpo_sweep_neartie} report the hyperparameter sweeps used for selection. For TV-DR-DPO, $\rho=0.10$ gives the best Kendall's $\tau$. For KLDPO, $\tau=1.00$ gives the best Kendall's $\tau$ and NDCG, while smaller temperatures over-concentrate on high-loss samples.

\begin{table}[h]
\centering
\caption{Hyperparameter sweep for TV-DR-DPO on Qwen3-0.6B under near-tie 0.4 noise.}
\label{tab:tvdpo_sweep_neartie}
\setlength{\tabcolsep}{6pt}
\begin{tabular}{lcccc}
\toprule
$\rho$ & Top-1 $\uparrow$ & Exact $\uparrow$ & Kendall's $\tau \uparrow$ & NDCG $\uparrow$ \\
\midrule
0.05 & \textbf{0.377} & \textbf{0.113} & 0.261 & 0.867 \\
0.10 & 0.367 & \textbf{0.113} & \textbf{0.270} & \textbf{0.869} \\
0.20 & 0.373 & 0.083 & 0.251 & \textbf{0.869} \\
0.40 & 0.357 & 0.080 & 0.209 & 0.866 \\
0.80 & 0.320 & 0.047 & 0.119 & 0.839 \\
\bottomrule
\end{tabular}
\end{table}

\begin{table}[h]
\centering
\caption{Hyperparameter sweep for KLDPO on Qwen3-0.6B under near-tie 0.4 noise.}
\label{tab:kldpo_sweep_neartie}
\setlength{\tabcolsep}{6pt}
\begin{tabular}{lcccc}
\toprule
$\tau$ & Top-1 $\uparrow$ & Exact $\uparrow$ & Kendall's $\tau \uparrow$ & NDCG $\uparrow$ \\
\midrule
0.05 & 0.343 & 0.083 & 0.206 & 0.860 \\
0.10 & 0.320 & 0.090 & 0.233 & 0.866 \\
0.20 & 0.353 & 0.113 & 0.242 & 0.868 \\
0.50 & \textbf{0.370} & \textbf{0.117} & 0.262 & 0.869 \\
1.00 & 0.363 & 0.097 & \textbf{0.267} & \textbf{0.871} \\
\bottomrule
\end{tabular}
\end{table}

\paragraph{Clean-label sanity checks.}
The noisy development sweeps above are used for hyperparameter selection. For completeness, Tables~\ref{tab:tvdpo_sweep_clean} and~\ref{tab:kldpo_sweep_clean} report clean-label sweeps for the same pairwise robust-DPO baselines. These clean sweeps are not used to select the main-table hyperparameters. Instead, they diagnose whether aggressive loss-level reweighting harms performance when the observed labels are reliable. The results show that large TV radii and small KL temperatures over-concentrate on high-loss samples and degrade clean ranking quality.

\begin{table}[h]
\centering
\small
\caption{Clean-label sanity sweep for TV-DR-DPO on Qwen3-0.6B.}
\label{tab:tvdpo_sweep_clean}
\setlength{\tabcolsep}{6pt}
\begin{tabular}{lcccc}
\toprule
$\rho$ & Top-1 $\uparrow$ & Exact $\uparrow$ & Kendall's $\tau \uparrow$ & NDCG $\uparrow$ \\
\midrule
0.05 & \textbf{0.407} & \textbf{0.140} & 0.281 & \textbf{0.879} \\
0.10 & \textbf{0.407} & 0.117 & \textbf{0.282} & 0.876 \\
0.20 & 0.387 & 0.123 & 0.259 & 0.874 \\
0.40 & 0.387 & 0.103 & 0.232 & 0.868 \\
0.80 & 0.320 & 0.083 & 0.189 & 0.859 \\
\bottomrule
\end{tabular}
\end{table}

\begin{table}[H]
\centering
\small
\caption{Clean-label sanity sweep for KLDPO on Qwen3-0.6B.}
\label{tab:kldpo_sweep_clean}
\setlength{\tabcolsep}{6pt}
\begin{tabular}{lcccc}
\toprule
$\tau$ & Top-1 $\uparrow$ & Exact $\uparrow$ & Kendall's $\tau \uparrow$ & NDCG $\uparrow$ \\
\midrule
0.05 & 0.370 & 0.093 & 0.241 & 0.871 \\
0.10 & 0.347 & 0.073 & 0.212 & 0.862 \\
0.20 & \textbf{0.393} & 0.097 & 0.231 & 0.870 \\
0.50 & 0.390 & 0.117 & 0.276 & \textbf{0.873} \\
1.00 & 0.387 & \textbf{0.120} & \textbf{0.277} & 0.872 \\
\bottomrule
\end{tabular}
\end{table}

Together, these sweeps show that the useful regime for pairwise robust-DPO baselines is mild loss-level reweighting. More aggressive settings, such as large $\rho$ for TV-DR-DPO or small $\tau$ for KLDPO, can overemphasize high-loss comparisons and hurt ranking quality even when the training labels are clean.

\subsection{Sensitivity of Robust PL to the Robustness Radius}
\label{app:rho_sensitivity}

To make the effect of the robustness strength explicit, we report full sweeps over the
robustness coefficient $\rho$.  Small positive values can stabilize learning under ranking-label noise, while large
values place excessive weight on the adversarial ranking in \eqref{eq:tv-decomp}.

\begin{table}[h]
\centering
\small
\caption{Sensitivity analysis of robustness coefficient $\rho$ on Qwen3-0.6B under near-tie 0.4 noise.}
\label{tab:rho_sweep_qwen3_neartie}
\setlength{\tabcolsep}{6pt}
\begin{tabular}{lcccc}
\toprule
$\rho$ & Top-1 $\uparrow$ & Exact $\uparrow$ & Kendall's $\tau \uparrow$ & NDCG $\uparrow$ \\
\midrule
0.00 (PL) & \textbf{0.407} & 0.110 & 0.268 & \textbf{0.874} \\
\midrule
0.05 & \textbf{0.407} & \textbf{0.150} & \textbf{0.274} & \textbf{0.874} \\
0.10 & 0.383 & 0.123 & 0.264 & 0.869 \\
0.15 & 0.393 & 0.107 & 0.256 & 0.869 \\
0.20 & 0.367 & 0.120 & 0.247 & 0.866 \\
\midrule
0.30 & 0.373 & 0.093 & 0.226 & 0.862 \\
0.50 & 0.323 & 0.060 & 0.103 & 0.841 \\
0.70 & 0.283 & 0.083 & 0.033 & 0.820 \\
1.00 & 0.237 & 0.040 & -0.053 & 0.794 \\
\bottomrule
\end{tabular}
\end{table}

\begin{table}[h]
\centering
\small
\caption{Sensitivity analysis of robustness coefficient $\rho$ on Qwen3-0.6B under clean labels.}
\label{tab:rho_sweep_qwen3_clean}
\setlength{\tabcolsep}{6pt}
\begin{tabular}{lcccc}
\toprule
$\rho$ & Top-1 $\uparrow$ & Exact $\uparrow$ & Kendall's $\tau \uparrow$ & NDCG $\uparrow$ \\
\midrule
0.00 (PL) & \textbf{0.410} & \textbf{0.123} & \textbf{0.291} & \textbf{0.879} \\
\midrule
0.05 & 0.400 & 0.113 & 0.282 & 0.875 \\
0.10 & 0.390 & 0.107 & 0.261 & 0.871 \\
0.15 & 0.390 & 0.107 & 0.269 & 0.873 \\
0.20 & 0.400 & 0.120 & 0.268 & 0.874 \\
\midrule
0.30 & 0.337 & 0.090 & 0.210 & 0.859 \\
0.50 & 0.333 & 0.063 & 0.126 & 0.844 \\
0.70 & 0.297 & 0.047 & 0.043 & 0.820 \\
1.00 & 0.240 & 0.060 & -0.054 & 0.787 \\
\bottomrule
\end{tabular}
\end{table}

\subsection{Additional Qwen2.5 offline fixed-list results}
\label{app:qwen25-offline}

\Cref{tab:main-05b} and \Cref{tab:main-7b} report full offline metrics on Qwen2.5. These results
provide metric-level support for the main-text offline story: clean-label performance is preserved,
while robust PL becomes more useful when ranking labels are corrupted.

\begin{table}[h]
\centering
\small
\caption{Sensitivity analysis of robustness coefficient $\rho$ on Qwen2.5-0.5B under top-rank 0.4 noise.}
\label{tab:rho-sweep}
\setlength{\tabcolsep}{6pt}
\begin{tabular}{lcccc}
\toprule
$\rho$ & Kendall's $\tau\uparrow$ & Top-1 $\uparrow$ & NDCG $\uparrow$ & PairAcc(K4) $\uparrow$ \\
\midrule
0.00 (PL) & 0.203 & 0.337 & 0.857 & 0.594 \\
\midrule
0.05 & \textbf{0.226} & \textbf{0.353} & \textbf{0.863} & \textbf{0.605} \\
0.10 & 0.206 & 0.317 & 0.852 & 0.595 \\
0.15 & 0.219 & 0.340 & 0.860 & 0.602 \\
0.20 & 0.184 & 0.330 & 0.852 & 0.584 \\
\midrule
0.30 & 0.139 & 0.300 & 0.840 & 0.562 \\
0.50 & 0.092 & 0.283 & 0.828 & 0.538 \\
0.70 & $-0.036$ & 0.240 & 0.801 & 0.474 \\
1.00 & $-0.034$ & 0.210 & 0.796 & 0.475 \\
\bottomrule
\end{tabular}
\end{table}

\begin{table}[h]
\centering
\footnotesize
\caption{Main results on Qwen2.5-0.5B across clean and noisy settings. Best results within each row are bold among available values.}
\label{tab:main-05b}
\setlength{\tabcolsep}{4pt}
\begin{tabular}{llcccc}
\toprule
Setting & Metric & Nominal BT & Nominal PL & Robust PL ($\rho=0.10$) & Robust PL ($\rho=0.05$) \\
\midrule
Clean & Top-1 $\uparrow$ & \textbf{0.376} & 0.374 & 0.374 & \textbf{0.376} \\
Clean & Exact $\uparrow$ & 0.088 & 0.098 & 0.096 & \textbf{0.100} \\
Clean & Kendall's $\tau \uparrow$ & 0.239 & \textbf{0.251} & \textbf{0.251} & 0.249 \\
Clean & NDCG $\uparrow$ & 0.874 & \textbf{0.875} & 0.873 & 0.874 \\
Clean & PairAcc(K4) $\uparrow$ & 0.616 & \textbf{0.622} & \textbf{0.622} & \textbf{0.622} \\
\midrule
Near-tie 0.4 & Top-1 $\uparrow$ & 0.313 & 0.337 & 0.360 & \textbf{0.363} \\
Near-tie 0.4 & Kendall's $\tau \uparrow$ & 0.206 & 0.225 & 0.235 & \textbf{0.266} \\
Near-tie 0.4 & NDCG $\uparrow$ & 0.865 & 0.868 & 0.867 & \textbf{0.869} \\
Near-tie 0.4 & PairAcc(K4) $\uparrow$ & 0.603 & 0.617 & 0.623 & \textbf{0.625} \\
\midrule
Near-tie 1.0 & Top-1 $\uparrow$ & 0.333 & 0.357 & 0.347 & \textbf{0.370} \\
Near-tie 1.0 & Kendall's $\tau \uparrow$ & 0.209 & 0.236 & 0.224 & \textbf{0.251} \\
Near-tie 1.0 & NDCG $\uparrow$ & 0.860 & 0.866 & 0.861 & \textbf{0.868} \\
Near-tie 1.0 & PairAcc(K4) $\uparrow$ & 0.598 & \textbf{0.618} & 0.611 & \textbf{0.618} \\
\midrule
Top-rank 0.4 & Top-1 $\uparrow$ & 0.323 & 0.353 & \textbf{0.357} & 0.327 \\
Top-rank 0.4 & Kendall's $\tau \uparrow$ & 0.199 & 0.214 & 0.219 & \textbf{0.232} \\
Top-rank 0.4 & NDCG $\uparrow$ & 0.851 & 0.859 & 0.861 & \textbf{0.863} \\
Top-rank 0.4 & PairAcc(K4) $\uparrow$ & 0.592 & 0.600 & 0.602 & \textbf{0.608} \\
\midrule
Top-rank 1.0 & Top-1 $\uparrow$ & 0.267 & 0.260 & 0.300 & \textbf{0.307} \\
Top-rank 1.0 & Kendall's $\tau \uparrow$ & $-0.016$ & 0.051 & 0.079 & \textbf{0.082} \\
Top-rank 1.0 & NDCG $\uparrow$ & 0.801 & 0.814 & \textbf{0.824} & 0.823 \\
Top-rank 1.0 & PairAcc(K4) $\uparrow$ & 0.484 & 0.518 & 0.532 & \textbf{0.533} \\
\bottomrule
\end{tabular}
\end{table}

\begin{table}[h]
\centering
\footnotesize
\caption{Results on Qwen2.5-7B under clean and noisy settings.}
\label{tab:main-7b}
\setlength{\tabcolsep}{5pt}
\begin{tabular}{llcccc}
\toprule
Setting & Metric & Nominal BT & Nominal PL & Robust ($\rho=0.10$) & Robust ($\rho=0.05$) \\
\midrule
Clean & Kendall's $\tau \uparrow$ & 0.333 & \textbf{0.390} & 0.376 & 0.380 \\
Clean & Top-1 $\uparrow$ & 0.452 & \textbf{0.501} & 0.480 & 0.483 \\
Clean & Exact $\uparrow$ & 0.140 & \textbf{0.171} & 0.160 & 0.163 \\
Clean & NDCG $\uparrow$ & 0.888 & \textbf{0.904} & 0.900 & 0.901 \\
Clean & PairAcc(K4) $\uparrow$ & 0.663 & \textbf{0.692} & 0.685 & 0.687 \\
Clean & PairAcc(bin) $\uparrow$ & 0.696 & \textbf{0.736} & 0.735 & 0.725 \\
\midrule
Top-rank 0.4 & Kendall's $\tau \uparrow$ & 0.138 & 0.341 & 0.338 & \textbf{0.356} \\
Top-rank 0.4 & Top-1 $\uparrow$ & 0.319 & 0.438 & 0.438 & \textbf{0.439} \\
Top-rank 0.4 & Exact $\uparrow$ & 0.083 & 0.128 & 0.134 & \textbf{0.148} \\
Top-rank 0.4 & NDCG $\uparrow$ & 0.850 & 0.890 & 0.894 & \textbf{0.897} \\
Top-rank 0.4 & PairAcc(K4) $\uparrow$ & 0.566 & 0.666 & \textbf{0.678} & 0.675 \\
Top-rank 0.4 & PairAcc(bin) $\uparrow$ & 0.581 & 0.674 & 0.681 & \textbf{0.689} \\
\midrule
Near-tie 0.4 & Kendall's $\tau \uparrow$ & 0.365 & 0.368 & \textbf{0.370} & \textbf{0.370} \\
Near-tie 0.4 & Top-1 $\uparrow$ & 0.453 & 0.444 & 0.448 & \textbf{0.463} \\
Near-tie 0.4 & Exact $\uparrow$ & 0.137 & 0.139 & \textbf{0.145} & \textbf{0.145} \\
Near-tie 0.4 & NDCG $\uparrow$ & 0.901 & \textbf{0.903} & \textbf{0.903} & \textbf{0.903} \\
Near-tie 0.4 & PairAcc(K4) $\uparrow$ & 0.679 & 0.681 & \textbf{0.682} & \textbf{0.682} \\
Near-tie 0.4 & PairAcc(bin) $\uparrow$ & 0.720 & 0.736 & 0.722 & \textbf{0.739} \\
\bottomrule
\end{tabular}
\end{table}

\section{Experimental Details}
\label{app:exp-details}

\subsection{Models, Data, and Evaluation Setup}
\label{app:add_exp_details}

\paragraph{Base Models.}
We report below the HuggingFace repositories of the base language models adopted throughout our experiments:
\begin{itemize}
    \item \textbf{Qwen3-0.6B (0.6B parameters):}\\
    \url{https://huggingface.co/Qwen/Qwen3-0.6B}

    \item \textbf{Qwen3-8B (8B parameters):}\\
    \url{https://huggingface.co/Qwen/Qwen3-8B}
\end{itemize}

\paragraph{Datasets.}
All training and evaluation data are obtained from publicly available preference datasets hosted on HuggingFace, see UltraFeedback~\citep{cui2023ultrafeedback}: \url{https://huggingface.co/datasets/openbmb/UltraFeedback}

\paragraph{Reward Model.}
For online reward-based optimization and offline reward evaluation, we use
\texttt{openbmb/Eurus-RM-7b} as the frozen reward model throughout our RLHF
experiments. The reward model is used \emph{as-is}, without any additional
fine-tuning in our work:
\url{https://huggingface.co/openbmb/Eurus-RM-7b}.

All RLHF experiments were conducted on a cluster of 8 NVIDIA RTX 4090 GPUs. The approximate training time is about 1 hour per model in the offline setting and about 3 hours per model in the online setting.

\subsection{Prompt Templates}
\label{app:prompt_templates}

We describe below the prompt template used in our experiments for offline evaluation on  the UltraFeedback dataset \citep{cui2023ultrafeedback}. 
Following the dataset authors, we adopt the official evaluation prompt template provided withUltraFeedback, which is also used during dataset construction.

In our setting, the prompt is designed to elicit detailed and constructive feedback for a given model response, along with an overall quality score. 
The evaluation focuses on multiple aspects of response quality, including helpfulness, truthfulness, honesty, and adherence to the given instruction.
We use this prompt template consistently across all methods to ensure a fair and controlled comparison.

\subsection{Assets and licenses.}
We use only publicly available datasets and models. 
Table~\ref{tab:assets_licenses} summarizes the main assets used in our experiments, together with their licenses or terms of use where applicable.

\begin{table}[h]
\centering
\small
\caption{Main existing assets used in the experiments.}
\label{tab:assets_licenses}
\begin{tabular}{lll}
\toprule
Asset & Usage & License / Terms \\
\midrule
UltraFeedback & Offline and online preference data & MIT License \\
Qwen3-0.6B / Qwen3-8B & Base language models & Apache License 2.0 \\
Eurus-RM-7B & Reward-model scoring and ranking & Apache License 2.0 \\
GPT-4 & LLM-as-a-judge evaluation & OpenAI API Terms of Use \\
Our released code & Reproduction of experiments &  MIT License \\
\bottomrule
\end{tabular}
\end{table}

\begin{table}[t]
\label{tab:ultrafeedback_prompt}
\centering
\renewcommand{\arraystretch}{1.3}
\setlength{\tabcolsep}{8pt}
\begin{tabular}{|p{4cm}|p{10cm}|}
\hline
\multicolumn{2}{|c|}{\textbf{Overall Score and Feedback Evaluation Prompt Template on UltraFeedback}} \\
\hline
\textbf{System Prompt:} 
& You are an AI assistant that helps people find information. \\
\hline
\textbf{User Prompt:} 
& Given my answer to an instruction, your role is to provide specific and constructive feedback for me. You should find the best way for me to learn from your feedback and improve my performance. \\

& You should consider multiple aspects of my answer, including helpfulness, truthfulness, honesty, and to what extent the answer follows instructions. \\

& \textbf{Instruction:} \\
& \{prompt\} \\

& \textbf{Answer:} \\
& \{answer\} \\

& Please act as a teacher and provide specific and constructive feedback. Besides describing the weaknesses of the answer, you should also provide specific suggestions to guide me toward understanding how to improve. \\

& Please note, however, that your suggestions should help me better complete the instructions, but you should not introduce new requirements that are not mentioned in the instructions. \\

& Your feedback should focus on enhancing my ability to think critically and respond accurately. However, never explicitly provide the reference answer, nor do polite phrases be required. \\

& Only respond with concise feedback in chat style. Finally, score the overall quality of the answer from 1 to 10, where 1 is the worst and 10 is the best. \\
\hline
\textbf{Format:} 
& \textbf{Feedback:} \\
& [Your feedback] \\

& \textbf{Overall Score:} \\
& [1--10] \\
\hline
\end{tabular}
\end{table}



\end{document}